%% file: main.tex
\RequirePackage[svgnames,table]{xcolor}

\documentclass[11pt,letterpaper,logo]{yalearxiv}

\input{preamble}
\input{command}

\title{Kalman Delta Networks: Uncertainty-aware Associative Memory}

\runningtitle{Kalman Delta Networks}

\usepackage{fontawesome5}
\usepackage{colortbl}

\definecolor{yaleblue}{RGB}{0,58,112}
\newcommand{\yale}{%
  \textsuperscript{%
    {\usefont{T1}{pbk}{m}{n}\textcolor{yaleblue}{\textbf{Y}}}%
  }%
}

\author{Ngoc Bui\yale, Tinglin Huang\yale, Rex Ying\yale\\
\yale Department of Computer Science, Yale University\\
\faGithub~\textbf{Source Code:} \href{https://github.com/ngocbh/kalman-delta-networks}{\texttt{https://github.com/ngocbh/kalman-delta-networks}}
}

\hypersetup{colorlinks=true, linkcolor=blue!50!black, citecolor=blue!50!black,
            urlcolor=blue!50!black}

\begin{abstract}
\vspace{-1mm}
{\centering\section*{Abstract}}
Linear attention enables efficient long-context inference by compressing token history into a fixed-size recurrent memory. This compression makes each update a trade-off between incorporating new information and preserving useful associations. Models such as DeltaNet, Gated DeltaNet, and KDA predict write strength from the current token representation, without explicitly tracking uncertainty in the stored memory. Yet this uncertainty matters: a new observation should have greater influence when the existing association is uncertain and less when it is already well supported. We introduce \textit{Kalman Delta Networks} (KDNs), a family of linear-attention models that explicitly track memory uncertainty to guide each update. By formulating associative memory as a linear–Gaussian state-space model, KDNs propagate both the memory estimate and its uncertainty, using the Kalman gain to balance accumulated evidence against the reliability of new observations. This formulation also recovers standard delta-rule updates by replacing tracked covariance with a token-predicted isotropic surrogate. To support hardware-efficient training and inference, we derive \textit{Diagonal KDN} and \textit{Isotropic KDN}, which retain one uncertainty value per key channel and per head, respectively. Their uncertainty updates admit associative scans with logarithmic parallel depth, requiring only $O(d_k)$ and $O(1)$ auxiliary state per head. Across controlled pretraining at 750M and 1.3B parameters, both variants consistently improve perplexity and mean downstream accuracy over the evaluated state-of-the-art linear-attention baselines.
\end{abstract}

\begin{document}

\maketitle

\section{Introduction}
\label{sec:intro}

Self-attention enables flexible retrieval from the preceding context by allowing each query to
aggregate the values associated with earlier keys~\citep{vaswani2017attention}. This
flexibility, however, requires retaining the full token history, so the autoregressive key--value
cache grows with context length while query--key interactions remain quadratic in sequence length.
Linear attention instead compresses the prefix into a fixed-size recurrent state that can be
updated online and evaluated in parallel by a scan~\citep{katharopoulos2020transformers}. This
efficiency turns attention into an online memory-management problem: each token must edit a
compressed summary of the past before the model knows which associations future queries will need.
An update that is too timid preserves stale information; one that is too aggressive destroys useful
memory.

Delta-rule mixers address this problem with a more selective update: they first read the memory's prediction
for the current key, then write only the residual~\citep{schlag2021linear,yang2024parallelizing}.
This update is commonly interpreted as a single online gradient step on the fast-weight memory's
instantaneous squared prediction loss. Gated DeltaNet~\citep{yang2024gdn} and
KDA~\citep{kimiteam2025kimilinear} retain this residual write while adding Mamba-style state-space
transitions~\citep{gu2023mamba,dao2024transformers} that decay stored content at scalar or
channel-wise rates. Their transition and write
are therefore typically understood through different lenses: state-space dynamics and online
optimization, respectively. More importantly, their write strength is predicted from the current
token representation alone, without accounting for the memory's confidence in the stored association.
Without tracking this uncertainty, they lack an explicit basis for deciding whether a large residual
should revise a tentative association or be discounted to protect a well-supported association.
This risks overwriting reliable associations or leaving uncertain ones insufficiently corrected.

\textbf{In this paper,} we provide a unified view of the recurrent updates in delta-rule mixers through a
linear--Gaussian state-space formulation, under which uncertainty emerges as their missing state
variable. Specifically, we model the recurrent memory as an estimate of a latent, non-stationary
key--value map. Each token supplies a noisy observation of that map at one key, while the learned
transition describes how the map persists or drifts between observations. Under linear--Gaussian
assumptions, the Kalman filter~\citep{kalman1960new} is the optimal recursive estimator that
propagates both the memory estimate and its covariance, balancing memory uncertainty against
observation noise to determine the correction gain. The resulting update retains the 
residual delta-write form while assigning greater
weight to new evidence when the memory is uncertain and protecting associations supported by
repeated observations. This formulation leads to \emph{Kalman Associative Memory}, which
generalizes delta-rule mixers through explicit covariance tracking. DeltaNet, Gated DeltaNet, and
KDA thus emerge as covariance-free approximations that retain different memory transitions while omitting this uncertainty recursion.

The exact Kalman update, however, is not a practical linear-attention layer. It
carries a dense $d_k\times d_k$ covariance per head, and its gain depends on a state-dependent
Riccati recurrence that is poorly suited to efficient parallel scans. To make uncertainty-aware
filtering practical, we introduce \emph{Kalman Delta Networks (KDNs)}, a family of scan-compatible
approximations to Kalman Associative Memory. We develop two special cases. \emph{Diagonal KDN}
imposes diagonal structure on the transition, process noise, and predictive covariance. Because
each exact update still yields a dense posterior covariance, we derive an online mean-field
variational update that projects the posterior back onto the diagonal family by minimizing reverse
KL. Conditional on the retained diagonal predictive prior, this projection preserves the exact
one-step posterior mean. To mitigate excessive overwrite under this approximation, we introduce an
information-scaling factor that calibrates the retained uncertainty. Diagonal KDN maintains one
uncertainty value per key channel and requires $O(d_k)$ auxiliary state per head.
In addition, we introduce \emph{Isotropic KDN}, which uses one uncertainty scalar per head,
trading expressiveness for $O(1)$ auxiliary state and minimal overhead compared to KDA.

Empirically, under parameter-matched recurrent-only pretraining on
FineWeb-Edu~\citep{penedo2024fineweb}, both KDN variants achieve lower WikiText and LAMBADA
perplexity and higher mean six-task zero-shot accuracy than every evaluated state-of-the-art
linear-time recurrent mixer---including Mamba-3~\citep{lahoti2026mamba3},
KDA~\citep{kimiteam2025kimilinear}, and GDN-2~\citep{hatamizadeh2026gdn2}---at both 750M/50B
and 1.3B/100B. Diagonal KDN also achieves the highest observed 14-cell RULER
aggregate~\citep{hsieh2024ruler} at both scales. Our contributions are:

\begin{itemize}[leftmargin=1.4em,itemsep=3pt,topsep=2pt]
  \item \textbf{A principled state-space view of delta-rule models.} We cast the delta rule as an
  innovation update in a linear--Gaussian state-space model, connecting delta-based recurrent
  mixers to the Mamba lineage~\citep{gu2023mamba}. DeltaNet, Gated DeltaNet, and KDA use
  identity, scalar, and diagonal transitions, respectively; unlike Mamba's control-driven additive
  input write, they correct a predicted key--value map with a key-conditioned residual, while using
  a token-predicted rather than covariance-derived gain.
  \item \textbf{A hardware-efficient algorithm for uncertainty-aware associative memory.} We derive Diagonal KDN through online
  variational inference and introduce Isotropic and Diagonal variants with $O(1)$ and $O(d_k)$
  uncertainty state per head, respectively. Their uncertainty updates admit an associative gain
  scan followed by the usual affine memory scan.
  \item \textbf{Overwrite analysis and empirical evidence.} We identify how the diagonal
  uncertainty approximation can underprotect stored key directions, introduce information scaling to control future
  overwrite, and show improvements over the evaluated state-of-the-art delta-rule models and
  Mamba-3 variants in both reported perplexities and mean six-task zero-shot accuracy.
\end{itemize}

\section{Preliminaries}
\label{sec:preliminaries}

\paragraph{From self-attention to recurrent memory.}
Causal self-attention~\citep{vaswani2017attention} retains all preceding key--value pairs for
flexible retrieval, whereas linear attention~\citep{katharopoulos2020transformers} compresses the
prefix into a fixed-size matrix $\ve{S}_t\in\R^{d_k\times d_v}$. Its write and read operations are:
\begin{equation}
  \text{(write)}\quad \ve{S}_t=\ve{S}_{t-1}+\ve{k}_t\ve{v}_t^{\tr},
  \qquad
  \text{(read)}\quad \ve{o}_t=\ve{S}_t^{\tr}\ve{q}_t.
  \label{eq:linear-attention-main}
\end{equation}
Queries read from this fixed-size state, avoiding a growing cache. Its additive write, however,
cannot revise or forget stored associations.

\paragraph{Delta-rule mixers.}
Delta-rule mixers address this limitation by first predicting the memory through a transition
$\ve{D}_t$ and then correcting its prediction at the current key with write strength $\beta_t$
\citep{schlag2021linear,yang2024parallelizing}:
\begin{equation}
\begin{aligned}
  \widehat{\ve{S}}_t
  &=\ve{D}_t\ve{S}_{t-1},
  &\quad \ve{S}_t
  &=\widehat{\ve{S}}_t
    +\beta_t\ve{k}_t
      \big(\ve{v}_t-\widehat{\ve{S}}_t^{\tr}\ve{k}_t\big)^{\tr}\\
  &&&=(\ve{I}-\beta_t\ve{k}_t\ve{k}_t^{\tr})\ve{D}_t\ve{S}_{t-1}
    +\beta_t\ve{k}_t\ve{v}_t^{\tr}.
\end{aligned}
\label{eq:unified-delta-rule}
\end{equation}
DeltaNet uses $\ve{D}_t=\ve{I}$, making writes selective but leaving untouched associations
unchanged. Gated DeltaNet adds scalar decay $\ve{D}_t=\alpha_t\ve{I}$ to forget stale content
globally~\citep{yang2024gdn}; KDA generalizes this to
$\ve{D}_t=\diag(\ve{\alpha}_t)$, allowing different key channels to retain information at different
rates~\citep{kimiteam2025kimilinear}. The progression is therefore from local residual correction,
to global forgetting, to channel-wise forgetting. In all three models, however, the scalar
$\beta_t$ is predicted from the current token rather than derived from confidence in the stored
association. Appendix~\ref{app:attention-ssm-background} reviews the standard attention and
state-space formulations.

\paragraph{Notation.}
For matrices $\ve{X}=[\ve{x}_1,\ldots,\ve{x}_{d_v}]$ and
$\ve{M}=[\ve{m}_1,\ldots,\ve{m}_{d_v}]$, we write
\begin{equation}
  \ve{X}\sim\N_{\mathrm{col}}(\ve{M},\ve{P})
  \quad\Longleftrightarrow\quad
  \ve{x}_j\sim\N(\ve{m}_j,\ve{P})
  \ \text{independently for }j=1,\ldots,d_v.
  \label{eq:columnwise-gaussian-notation}
\end{equation}

\section{Kalman Associative Memory}
\label{sec:memory-kalman}

We revisit DeltaNet-style models through the lens of a linear--Gaussian state-space model, for
which the optimal recursive estimator is the Kalman filter~\citep{kalman1960new}. Under this
view, DeltaNet, Gated DeltaNet, and KDA emerge as fixed-gain special cases: they retain the Kalman
residual correction
but omit covariance tracking and predict the write strength directly from the current input.
Forgetting is interpreted as the transition model's prediction of how the latent memory state
persists or changes between observations. Tracking the covariance then quantifies confidence in
the predicted memory and yields a principled gain for incorporating new evidence.

\subsection{Associative Memory as a Linear-Gaussian State-Space Model}
\label{sec:memory-formulation}

To make this filtering view precise, let the memory at time $t$ be a matrix
$\ve{S}_t\in\R^{d_k\times d_v}$ defining the associative map
\begin{equation}
  \mc M_t\colon \R^{d_k}\to\R^{d_v},\qquad \mc M_t(\ve{k})=\ve{S}_t^{\tr}\ve{k} .
  \label{eq:assoc-map}
\end{equation}
The query read is $\ve{o}_t=\mc M_t(\ve{q}_t)=\ve{S}_t^{\tr}\ve{q}_t$. The streaming token provides one
supervised measurement of this map: under key $\ve{k}_t$, the memory should return value $\ve{v}_t$,
i.e.\ $\mc M_t(\ve{k}_t)\approx \ve{v}_t$. This is the shape of a state-space
model~\citep{gu2023mamba}, in which an unobserved state evolves over time, and each token reveals a
noisy linear projection of it.

\paragraph{Latent state.}
Let $\widetilde{\ve{S}}_t$ be the latent associative map the recurrent state is trying to track. The
transition describes how stored associations persist, decay, or drift before the next
measurement arrives:
\begin{equation}
  \widetilde{\ve{S}}_t
  = \ve{D}_t \widetilde{\ve{S}}_{t-1} + \ve{W}_t,
  \qquad
  \ve{W}_t\sim\N_{\mathrm{col}}(\ve{0},\ve{\Omega}_t).
  \label{eq:mem-update-dynamics}
\end{equation}
The operator $\ve{D}_t$ is the process model: it is the formal object corresponding to forgetting
under information drift, topic drift, and other non-stationarity in the stream. For DeltaNet
$\ve{D}_t=\ve{I}$; for Gated DeltaNet $\ve{D}_t=\alpha_t \ve{I}$; for KDA
$\ve{D}_t=\diag(\ve{\alpha}_t)$. The process noise $\ve{W}_t$ represents memory drift not captured
by deterministic decay, while $\ve{\Omega}_t$ quantifies uncertainty about that drift. In a
language model, this accounts for changes in the latent discourse state that cannot be inferred
from retention alone. If an incoming sentence reveals that an entity has moved from Paris to
Rome, $\ve{D}_t$ can attenuate the stale Paris binding, but only the observation can establish the
new Rome binding.
\paragraph{Observation.}
The current key reads one direction of the latent map. The observed value is
\begin{equation}
  \ve{v}_t = \widetilde{\ve{S}}_t^{\tr}\ve{k}_t + \ve{e}_t,
  \qquad
  \ve{e}_t\sim\N(0,r_t \ve{I}_{d_v}).
  \label{eq:mem-update-observation}
\end{equation}
Thus $\ve{k}_t^{\tr}$ is the observation map and $\ve{v}_t$ is the measurement. The noise term $\ve{e}_t$
captures the part of the token value that should \emph{not} be treated as a clean memory
target. In language modeling, $\ve{v}_t$ is a contextual feature produced from a noisy token embedding,
local syntax, position, layer mixing, and finite model capacity; not all of it is a stable
fact that can be stored under key $\ve{k}_t$. The covariance $r_t\ve{I}_{d_v}$ sets how reliable this value
observation is: small $r_t$ means trust the token and write strongly, while large $r_t$ means
treat much of $\ve{v}_t$ as context-specific noise.

\paragraph{Online filtering problem.}
We assume $\widetilde{\ve{S}}_0\sim\N_{\mathrm{col}}(\ve{0},\ve{I})$. Conditioned on the
input-dependent quantities $\{(\ve{D}_t,\ve{\Omega}_t,\ve{k}_t,r_t)\}_{t\geq1}$, the pairs
$\{(\ve{W}_t,\ve{e}_t)\}_{t\geq1}$ are independent across time and of
$\widetilde{\ve{S}}_0$, with $\ve{W}_t\perp\ve{e}_t$ at each step.
Given the observation history through step $t$,
$\mathcal{F}_t=\sigma\{(\ve{k}_\tau,\ve{v}_\tau):\tau\leq t\}$, we estimate the current latent
memory by its posterior mean, $\ve{S}_t=\E[\widetilde{\ve{S}}_t\mid\mathcal{F}_t]$.

\subsection{Kalman Filtering}
\label{sec:kalman-filtering}

Under the linear--Gaussian assumptions in \eqref{eq:mem-update-dynamics} and
\eqref{eq:mem-update-observation}, the Kalman filter~\citep{kalman1960new} computes this posterior mean exactly by first
predicting the latent memory and then conditioning that prediction on the current token. It carries
the filtering law
$\widetilde{\ve{S}}_t\mid\mathcal{F}_t
\sim\N_{\mathrm{col}}(\ve{S}_t,\ve{P}_t)$, where
$\ve{S}_t=\E[\widetilde{\ve{S}}_t\mid\mathcal{F}_t]$ and
$\ve{P}_t\in\R^{d_k\times d_k}$ is the common key-space covariance.

Tracking $\ve{P}_t$ is necessary because each token observes the associative map only at one key
direction $\ve{k}_t$, and the value $\ve{v}_t$ is noisy. After many tokens, some key directions are
well supported by past measurements, while others remain uncertain or have become stale
because of drift. The filter uses this uncertainty to calibrate each write, allowing larger
corrections in uncertain directions while protecting well-supported associations from noisy
measurements.

To derive the filtering recursion, suppose the previous posterior is
$\widetilde{\ve{S}}_{t-1}\mid\mathcal{F}_{t-1}
\sim\N_{\mathrm{col}}(\ve{S}_{t-1},\ve{P}_{t-1})$. The linear transition with Gaussian process
noise preserves this family, giving the predictive distribution
$\widetilde{\ve{S}}_t\mid\mathcal{F}_{t-1}
\sim\N_{\mathrm{col}}(\widehat{\ve{S}}_t,\widehat{\ve{P}}_t)$. The observation in
\eqref{eq:mem-update-observation} measures each predicted column along the same key $\ve{k}_t$ with
independent Gaussian noise, so the predicted state and $\ve{v}_t$ are jointly Gaussian.
Conditioning this joint Gaussian on $\ve{v}_t$ yields the residual correction, shared Kalman gain,
and covariance update summarized below.

\refstepcounter{definition}\label{prop:kalman-associative-memory}
\begin{tcolorbox}[
  title={Proposition~\thedefinition: Kalman Associative Memory},
  colback=blue!7,
  colbacktitle=blue!22,
  colframe=blue!45!black,
  coltitle=black,
  fonttitle=\bfseries,
  boxrule=0.4pt,
  arc=1pt,
  left=5pt,
  right=5pt,
  top=5pt,
  bottom=5pt
]
Under the linear--Gaussian model in
\eqref{eq:mem-update-dynamics}--\eqref{eq:mem-update-observation}, the Kalman filter gives the
exact posterior mean. It predicts memory and uncertainty from
$(\ve{S}_{t-1},\ve{P}_{t-1})$ using $(\ve{D}_t,\ve{\Omega}_t)$, then updates the memory with the
residual $\ve{v}_t-\widehat{\ve{S}}_t^{\tr}\ve{k}_t$ weighted by the uncertainty-dependent gain
$\ve{\kappa}_t$.
\begin{equation}
\begin{aligned}
  \text{predict:}\quad
  \widehat{\ve{S}}_t &= \ve{D}_t \ve{S}_{t-1},
  &
  \widehat{\ve{P}}_t &= \ve{D}_t \ve{P}_{t-1}\ve{D}_t^{\tr}+\ve{\Omega}_t,
  \\[3pt]
  \text{update:}\quad
  \ve{S}_t &= \widehat{\ve{S}}_t+\ve{\kappa}_t
  \big(\ve{v}_t-\widehat{\ve{S}}_t^{\tr}\ve{k}_t\big)^{\tr},
  &
  \ve{P}_t &= (\ve{I}-\ve{\kappa}_t \ve{k}_t^{\tr})\widehat{\ve{P}}_t .
  \\[3pt]
  \text{where}\quad
  \ve{\kappa}_t &= \frac{\widehat{\ve{P}}_t \ve{k}_t}
  {r_t+\ve{k}_t^{\tr}\widehat{\ve{P}}_t \ve{k}_t}.
\end{aligned}
\label{eq:kalman-optimal-update-box}
\end{equation}
\end{tcolorbox}

\noindent\textit{Proof.} See Appendix~\ref{app:proof-kalman-optimal-update}.

Thus the posterior mean is a residual write: the filter adds only the part of the observed value
that the predicted memory failed to explain. The corresponding covariance update records the
remaining uncertainty after this observation.

\begin{figure}[t]
  \vspace{-3mm}
  \centering
  \includegraphics[width=0.92\textwidth]{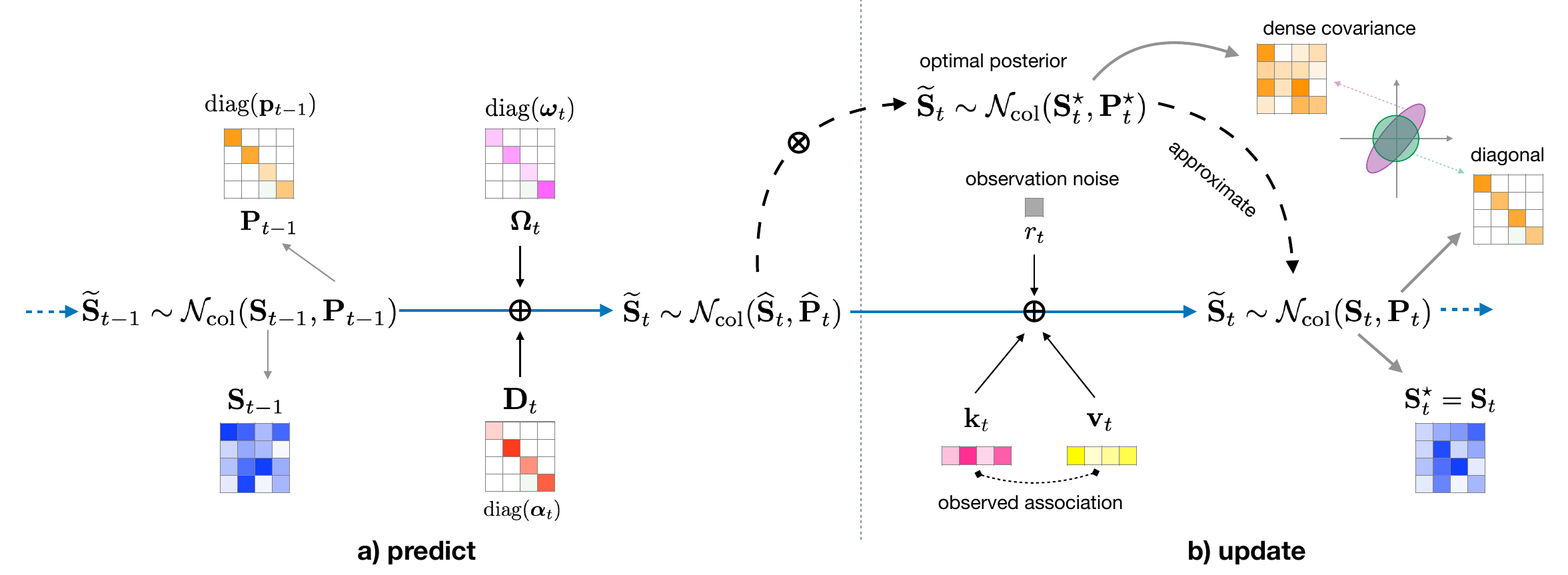}
  \vspace{-3mm}
  \caption{One Kalman Associative Memory update and its mean-field Diagonal KDN approximation.
  Prediction preserves diagonality; the exact update produces a dense posterior, which reverse-KL
  projection re-diagonalizes while preserving its mean. Here,
  $(\ve{S}_t^\star,\ve{P}_t^\star)$ is the optimal posterior from
  Proposition~\ref{prop:kalman-associative-memory}, and $(\ve{S}_t,\ve{P}_t)$ is the variational
  approximation from Proposition~\ref{prop:diag-variational-inference}.
  }
  \label{fig:kdn-variational-approximation}
  \vspace{-3mm}
\end{figure}

\subsection{Delta-rule models as fixed-gain Kalman filters}
\label{sec:delta-family-kalman}

The update above separates two modeling choices: the transition $\ve{D}_t$ predicts how the
memory changes before the token is written, and the gain $\ve{\kappa}_t$ decides the key-space
direction and strength of the residual write. Delta-rule linear attention is the
fixed-gain special case of this filter. If we discard the covariance recursion and use an
isotropic surrogate $\widehat{\ve{P}}_t\approx \widehat b_t \ve{I}$ for the predicted uncertainty, then
\begin{equation}
  \ve{\kappa}_t
  =
  \frac{\widehat{\ve{P}}_t \ve{k}_t}{r_t+\ve{k}_t^{\tr}\widehat{\ve{P}}_t \ve{k}_t}
  \approx
  \frac{\widehat b_t}{r_t+\widehat b_t\|\ve{k}_t\|_2^2}\,\ve{k}_t
  \equiv
  \beta_t \ve{k}_t .
  \label{eq:fixed-gain-approx}
\end{equation}
For normalized keys, the exact Kalman gain therefore collapses to the scalar write strength
$\beta_t$ used by the delta-rule family. Substituting $\ve{\kappa}_t=\beta_t \ve{k}_t$ into the
Kalman update gives the shared fixed-gain residual form. DeltaNet, Gated DeltaNet, and KDA then
differ through their process models. Relative to the Kalman optimal update, delta-rule models
freeze the uncertainty dynamics and replace the adaptive gain with the fixed
first-order direction
$\beta_t \ve{k}_t$.

\section{Kalman Delta Networks}
\label{sec:kla}
\vspace{-1mm}

The Kalman optimal update in \eqref{eq:kalman-optimal-update-box} is the target, but the exact
recursion is not directly compatible with fully parallel linear attention due to its Riccati
covariance recursion. This makes the gain depend on accumulated posterior uncertainty and
requires a dense $d_k\times d_k$ covariance state per head. We therefore seek a relaxation that
preserves both uncertainty-aware updates and parallelizability.

\subsection{Diagonal Kalman Delta Network}
\label{sec:diagonal-kdn}

Let $\mathcal{D}_{++}^{d_k}=\{\diag(\ve{p}):\ve{p}\in\R_{++}^{d_k}\}$.
We first restrict the covariance tracked by the filter:
\begin{equation}
  \ve{P}_{t-1}=\diag(\ve{p}_{t-1})\in\mathcal{D}_{++}^{d_k},
  \qquad
  \ve{D}_t=\diag(\ve{\alpha}_t),
  \qquad
  \ve{\Omega}_t=\diag(\ve{\omega}_t)\in\mathcal{D}_{++}^{d_k}.
  \label{eq:diag-cov-info}
\end{equation}
The covariance prediction then stays in the diagonal family:
\begin{equation}
  \widehat{\ve{P}}_t=\diag(\widehat{\ve{p}}_t),
  \qquad
  \widehat{\ve{p}}_t
  =\ve{\alpha}_t^2\odot\ve{p}_{t-1}+\ve{\omega}_t.
  \label{eq:diag-covariance-prediction}
\end{equation}
Conditioned on this diagonal predictive covariance, the Kalman gain is still exact:
\begin{equation}
  \ve{\kappa}_t =
  \frac{\widehat{\ve{p}}_t\odot \ve{k}_t}
  {r_t+\sum_i \widehat{\ve{p}}_{t,i}\ve{k}_{t,i}^2}.
  \label{eq:diag-kalman-gain}
\end{equation}
Although $\widehat{\ve{P}}_t$ is diagonal, the exact posterior covariance
$\ve{P}_t^\star=(\ve{I}-\ve{\kappa}_t\ve{k}_t^{\tr})\widehat{\ve{P}}_t$
is generally dense; thus it immediately leaves
$\mathcal{D}_{++}^{d_k}$.

\paragraph{Online variational inference.}
To keep the tracked covariance diagonal, we project the exact posterior after each token onto
the mean-field family
\begin{equation}
  \mathcal{Q}_{\mathrm{col}}
  =\left\{
    \N_{\mathrm{col}}(\ve{S},\diag(\ve{p})):
    \ve{S}\in\R^{d_k\times d_v},\ \ve{p}\in\R_{++}^{d_k}
  \right\}.
  \label{eq:diag-variational-family}
\end{equation}
Let $p(\widetilde{\ve{S}}_t\mid\ve{v}_t,\mathcal{F}_{t-1},\ve{k}_t,r_t)
=\N_{\mathrm{col}}(\ve{S}_t^\star,\ve{P}_t^\star)$ denote the exact one-step posterior given by
Proposition~\ref{prop:kalman-associative-memory}. The following proposition characterizes its
diagonal variational projection.

\begin{proposition}[Online diagonal variational update]
\label{prop:diag-variational-inference}
Under the linear--Gaussian model with diagonal predictive covariance, treating $\ve{k}_t$ and $r_t$ as known at step $t$,
the variational approximation
\begin{equation}
  q_t
  =\argmin_{q\in\mathcal{Q}_{\mathrm{col}}}
    \operatorname{KL}\!\left(
      q\,\|\,\N_{\mathrm{col}}(\ve{S}_t^\star,\ve{P}_t^\star)
    \right)
  \label{eq:diag-variational-objective}
\end{equation}
has the unique solution
\begin{equation}
  q_t
  =\N_{\mathrm{col}}\!\left(
      \ve{S}_t,\diag(\ve{p}_t)
    \right),
  \label{eq:diag-variational-posterior}
\end{equation}
where
\begin{equation}
\begin{aligned}
  \ve{S}_t
  &=\ve{S}_t^\star
    =\widehat{\ve{S}}_t
    +\ve{\kappa}_t
      \left(\ve{v}_t-\widehat{\ve{S}}_t^{\tr}\ve{k}_t\right)^{\tr},
  &
  p_{t,i}
  &=\left(\widehat p_{t,i}^{-1}+\frac{k_{t,i}^2}{r_t}\right)^{-1}.
\end{aligned}
  \label{eq:diag-variational-solution}
\end{equation}
Thus the projection preserves the exact one-step Kalman posterior mean conditional on the diagonal
predictive prior and replaces its shared dense key-space covariance with a diagonal state.
\end{proposition}

\noindent\textit{Proof.} See Appendix~\ref{app:proof-diag-variational-inference}.

Using this variational solution as the next token's prior gives an online mean-field filter.
Crucially, the posterior covariance remains in the diagonal family after every token:
$\ve{P}_t=\diag(\ve{p}_t)$. Defining the token's diagonal precision increment
$\ve{u}_t=(\ve{k}_t\odot\ve{k}_t)/r_t$, the covariance update is
\begin{equation}
  \widehat{\ve{p}}_t
  =\ve{\alpha}_t^2\odot\ve{p}_{t-1}+\ve{\omega}_t,
  \qquad
  \ve{p}_t
  =\left(\widehat{\ve{p}}_t^{-1}+\ve{u}_t\right)^{-1}
  =\frac{\widehat{\ve{p}}_t}
    {\ve{1}+\ve{u}_t\odot\widehat{\ve{p}}_t},
  \qquad
  \ve{p}_0=\ve{1},
  \label{eq:prefix-diag-uncertainty}
\end{equation}
where products, ratios, and inverses are coordinatewise. This approximation enables a M\"obius
associative scan of the covariance recurrence with logarithmic parallel depth, similar to that of
\citet{shaj2026kalmanlinearattention} (see Appendix~\ref{sec:chunkwise-gain}).
Figure~\ref{fig:kdn-variational-approximation}
summarizes this online predict--update--project cycle.

\begin{wrapfigure}{r}{0.48\textwidth}
  \vspace{-0.8\baselineskip}
  \vspace{-6mm}
  \centering
  \includegraphics[width=\linewidth]{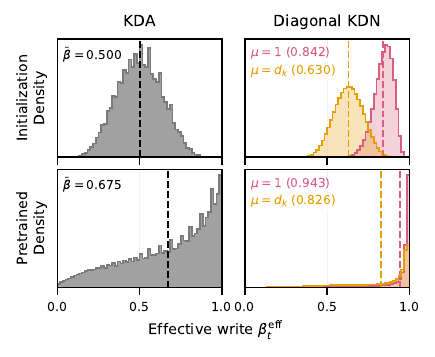}
  \vspace{-0.8\baselineskip}
  \vspace{-6mm}
  \caption{Effective-write distributions before (top) and after 50B-token pretraining (bottom).}
  \label{fig:kdn-overwrite}
\end{wrapfigure}
\paragraph{Information scaling.}
Diagonal projection discards correlations that encode confidence in combinations of key channels.
It can therefore overestimate uncertainty along an already observed key direction, causing repeated
or similar keys to trigger overly strong writes. We mitigate this by scaling the post-write precision
increment by $\mu>0$:
\begin{equation}
  \ve{u}_t
  =\frac{\mu}{r_t}\,\ve{k}_t\odot\ve{k}_t.
  \label{eq:diag-cov-update}
\end{equation}
Since $\ve{\kappa}_t$ is computed beforehand, increasing $\mu$ preserves the current write but
reduces later effective writes $\beta^{\mathrm{eff}}=\ve{k}^{\tr}\ve{\kappa}$.
Here $\mu=1$ recovers the variational update, while $\mu=d_k$ offsets the $1/d_k$ information
dilution for normalized dense keys (see Appendix~\ref{app:protection-write}).
In Figure~\ref{fig:kdn-overwrite}, increasing $\mu$ from $1$ to $d_k$ lowers the pooled mean
$\beta^{\mathrm{eff}}$.

The following proposition summarizes the resulting Diagonal Kalman Delta Network.

\refstepcounter{definition}\label{prop:kalman-linear-attention}
\begin{tcolorbox}[
  title={Proposition~\thedefinition: Diagonal Kalman Delta Network},
  breakable,
  colback=blue!7,
  colbacktitle=blue!22,
  colframe=blue!45!black,
  coltitle=black,
  fonttitle=\bfseries,
  boxrule=0.4pt,
  arc=1pt,
  left=5pt,
  right=5pt,
  top=5pt,
  bottom=5pt
]
Let the predictive covariance be restricted to $\mathcal{D}_{++}^{d_k}$ and let
$\ve{\alpha}_t$, $\ve{\omega}_t$, and $r_t$ be parameterized by
\begin{equation}
\begin{aligned}
  \ve{\alpha}_t &= \sigma(\ve{W}_{\alpha}\ve{x}_t+\ve{b}_{\alpha}),
  &
  \ve{\omega}_t &= \operatorname{softplus}(\ve{W}_{\omega}\ve{x}_t+\ve{b}_{\omega}),
  &
  r_t &= r_{\min}+\operatorname{softplus}(\ve{w}_{r}^{\tr}\ve{x}_t+b_r).
\end{aligned}
\label{eq:kla-box-parameterization}
\end{equation}
Here $\ve{x}_t$ is the token hidden state, and all weights $\ve{W}_{\ast}$ and
$\ve{w}_{\ast}$, together with their biases, are learnable parameters.
The covariance states are computed by the per-channel M\"obius scan
\begin{equation}
  \ve{M}_{t,i}
  =
  \begin{bmatrix}
    \alpha_{t,i}^2 & \omega_{t,i}\\
    u_{t,i}\alpha_{t,i}^2 & 1+u_{t,i}\omega_{t,i}
  \end{bmatrix},
  \begin{bmatrix}n_{t,i}\\ d_{t,i}\end{bmatrix}
  =
  \ve{M}_{t,i}
  \begin{bmatrix}n_{t-1,i}\\ d_{t-1,i}\end{bmatrix},
  p_{t,i}=\frac{n_{t,i}}{d_{t,i}},
  \begin{bmatrix}n_{0,i}\\ d_{0,i}\end{bmatrix}
  =
  \begin{bmatrix}1\\ 1\end{bmatrix}.
  \label{eq:kla-box-mobius}
\end{equation}
The adaptive gain at token $t$ uses the pre-write predictive covariance $\widehat{\ve{p}}_t$:
\begin{equation}
  \ve{\kappa}_t
  =
  \frac{\widehat{\ve{p}}_t\odot\ve{k}_t}
  {r_t+\sum_i \widehat{\ve{p}}_{t,i}\ve{k}_{t,i}^2}.
\label{eq:kla-box-gain}
\end{equation}
With $\ve{\kappa}_t$ available before the memory scan, the memory update is an affine scan:
\begin{equation}
\begin{aligned}
  \ve{S}_t
  &=
  (\ve{I}-\ve{\kappa}_t\ve{k}_t^{\tr})\diag(\ve{\alpha}_t)\ve{S}_{t-1}
  +\ve{\kappa}_t\ve{v}_t^{\tr},
  &
  \ve{o}_t &= \ve{S}_t^{\tr}\ve{q}_t .
\end{aligned}
\label{eq:kla-box-memory}
\end{equation}
\end{tcolorbox}

Thus the diagonal approximation preserves the two ingredients missing from fixed-gain
delta-rule models while keeping the training primitive unchanged: the memory is predicted
through $\ve{D}_t$, and the residual write is weighted by a transition-aware uncertainty
estimate.

\paragraph{Hardware-efficient chunkwise implementation.}
The diagonal approximation directly yields a two-stage parallel implementation. Its per-channel
covariance recurrence is a M\"obius map, enabling an associative chunkwise kernel that computes
all Kalman gains; once those gains are available, the memory update becomes an input-only
asymmetric delta rule handled by a compact-WY kernel. Appendix~\ref{sec:chunkwise-algorithm}
provides the derivation and implementation details.

\subsection{Isotropic Kalman Delta Network}
\label{sec:isotropic-kdn}

Instead of tracking a separate uncertainty for every key channel, we can further restrict the
covariance to the isotropic family $\ve{P}_t\approx b_t\ve{I}$. Unlike the diagonal family, the
isotropic family is not closed under a channel-wise transition, so Isotropic KDN projects both the
predicted covariance and the measurement posterior back to this family. For
$\ve{D}_t=\diag(\ve{\alpha}_t)$ and $\ve{\Omega}_t=\omega_t\ve{I}$, the trace-projected
prediction, uncertainty-derived gain, and memory update are
\begin{equation}
\begin{aligned}
  a_t
  &=\frac{1}{d_k}\sum_{i=1}^{d_k}\alpha_{t,i}^2,
  \quad
  \widehat b_t
  =a_t b_{t-1}+\omega_t,
  \quad
  \beta_t
  =\frac{\widehat b_t}{r_t+\widehat b_t\|\ve{k}_t\|_2^2},
  \\[3pt]
  \ve{S}_t
  &=(\ve{I}-\beta_t\ve{k}_t\ve{k}_t^{\tr})\ve{D}_t\ve{S}_{t-1}
    +\beta_t\ve{k}_t\ve{v}_t^{\tr}.
\end{aligned}
  \label{eq:main-iso-kdn-gain}
\end{equation}
This retains the scalar-gated form used by DeltaNet, Gated DeltaNet, and KDA, but $\beta_t$ now
depends on the transition, process uncertainty, and evidence accumulated from previous tokens
through $\widehat b_t$. It therefore replaces the independent token gate
$\sigma(\ve{w}_{\beta}^{\tr}\ve{x}_t+b_{\beta})$ with an uncertainty-derived write strength.
Appendix~\ref{app:iso-kla} derives the two isotropic projections and the posterior update.

\begin{table}[t]
\centering
\footnotesize
\setlength{\tabcolsep}{4.5pt}
\begin{tabular}{@{}l >{\columncolor{gray!10}}c >{\columncolor{gray!10}}c cccccc >{\columncolor{gray!10}}c@{}}
\toprule
& \multicolumn{2}{>{\columncolor{gray!10}}c}{Perplexity $\downarrow$}
& \multicolumn{6}{c}{Zero-shot accuracy (\%) $\uparrow$} & \cellcolor{gray!10}\\
\cmidrule(lr){2-3}\cmidrule(lr){4-9}
Model & Wiki. & LMB. & LMB. & PIQA & Hella. & Wino. & ARC-e & ARC-c & Avg. \\
\midrule
\multicolumn{10}{@{}l}{\emph{Recurrent-only, 750M parameters, 50B tokens}}\\
DeltaNet~\citeyearpar{yang2024parallelizing}         & 19.78 & 20.17 & 38.77 & 69.10 & 48.04 & 51.93 & 65.49 & 33.28 & 51.10 \\
Gated DeltaNet~\citeyearpar{yang2024gdn}             & 19.50 & 18.08 & 40.23 & 69.91 & 50.22 & 55.64 & \textbf{68.10} & 32.17 & 52.71 \\
KDA~\citeyearpar{kimiteam2025kimilinear}             & 18.85 & 15.06 & \underline{44.34} & \textbf{70.95} & 51.23 & 54.93 & 67.72 & 34.04 & 53.87 \\
Mamba-3 (SISO)~\citeyearpar{lahoti2026mamba3}  & 19.68 & 17.61 & 40.33 & \underline{70.57} & 50.87 & 53.83 & 67.51 & 34.47 & 52.93 \\
Mamba-3 (MIMO)~\citeyearpar{lahoti2026mamba3}  & 18.99 & 15.67 & 43.24 & 69.86 & \textbf{52.24} & \underline{56.75} & 67.72 & \textbf{36.52} & 54.39 \\
GDN-2~\citeyearpar{hatamizadeh2026gdn2}           & 21.20 & 17.88 & 41.18 & 70.08 & 46.90 & 54.54 & 64.27 & 31.74 & 51.45 \\
\rowcolor{blue!8}
Isotropic KDN   & \textbf{18.42} & \underline{14.68} & 44.05 & \textbf{70.95} & \underline{52.12} & 56.12 & \underline{68.01} & 35.24 & \underline{54.41} \\
\rowcolor{blue!8}
Diagonal KDN    & \underline{18.64} & \textbf{14.15} & \textbf{45.66} & \textbf{70.95} & 51.91 & \textbf{57.70} & \textbf{68.10} & \underline{35.49} & \textbf{54.97} \\
\midrule
\multicolumn{10}{@{}l}{\emph{Recurrent-only, 1.3B parameters, 100B tokens}}\\
KDA~\citeyearpar{kimiteam2025kimilinear}             & 15.40 & 10.09 & \underline{51.15} & \underline{74.16} & 60.61 & 60.54 & 73.48 & \textbf{41.72} & 60.28 \\
Mamba-3 (SISO)~\citeyearpar{lahoti2026mamba3}  & 15.94 & 11.98 & 47.45 & 73.61 & 59.39 & 57.77 & 72.64 & 38.31 & 58.20 \\
Mamba-3 (MIMO)~\citeyearpar{lahoti2026mamba3}  & 15.63 & 10.49 & 50.20 & 73.94 & \textbf{60.79} & 59.19 & \underline{73.91} & 41.04 & 59.85 \\
GDN-2~\citeyearpar{hatamizadeh2026gdn2}           & 16.15 & 11.29 & 49.45 & 72.63 & 57.80 & 59.19 & 72.73 & 39.25 & 58.51 \\
\rowcolor{blue!8}
Isotropic KDN   & \underline{15.30} & \underline{9.98} & 50.96 & 73.23 & 60.23 & \textbf{61.33} & \textbf{74.71} & \underline{41.64} & \underline{60.35} \\
\rowcolor{blue!8}
Diagonal KDN    & \textbf{15.04} & \textbf{9.75} & \textbf{51.87} & \textbf{74.21} & \underline{60.68} & \underline{60.62} & 73.70 & \underline{41.64} & \textbf{60.45} \\
\midrule
\multicolumn{10}{@{}l}{\emph{Hybrid and attention-only, 1.3B parameters, 100B tokens}}\\
Transformer~\citeyearpar{vaswani2017attention} (2K SWA)       & 16.67 & 13.10 & 48.38 & 71.22 & 56.62 & 56.75 & 68.56 & 35.84 & 56.23 \\
KDA~\citeyearpar{kimiteam2025kimilinear} $+$ SWA           & \underline{15.08} & \underline{10.81} & \underline{51.23} & 72.14 & \underline{60.33} & \textbf{61.64} & 72.77 & \underline{41.55} & \underline{59.94} \\
Mamba-3 (SISO)~\citeyearpar{lahoti2026mamba3} $+$ SWA & 15.87 & 11.38 & 50.46 & 72.80 & 59.59 & 59.27 & 72.73 & 40.10 & 59.16 \\
Mamba-3 (MIMO)~\citeyearpar{lahoti2026mamba3} $+$ SWA & 15.33 & 10.96 & 50.44 & 72.69 & 60.00 & 58.56 & \underline{72.81} & 41.21 & 59.28 \\
GDN-2~\citeyearpar{hatamizadeh2026gdn2} $+$ SWA         & 16.06 & 11.12 & 49.54 & 71.82 & 58.44 & 57.77 & 71.42 & 37.63 & 57.77 \\
\rowcolor{blue!8}
Isotropic KDN $+$ SWA & \textbf{14.98} & 10.90 & 50.49 & \underline{72.85} & 60.31 & 59.12 & 72.64 & 40.87 & 59.38 \\
\rowcolor{blue!8}
Diagonal KDN $+$ SWA  & 15.10 & \textbf{10.41} & \textbf{51.99} & \textbf{72.96} & \textbf{60.57} & \underline{60.38} & \textbf{72.98} & \textbf{41.81} & \textbf{60.11} \\
\bottomrule
\end{tabular}
\caption{Language modeling and zero-shot commonsense reasoning.}
\label{tab:pretrain-lm}
\end{table}

\Needspace{16\baselineskip}
\section{Experiments}
\label{sec:experiments}

\subsection{Experimental Setup}

\begin{wrapfigure}{R}{0.48\textwidth}
  \vspace{-21mm}
  \centering
  \includegraphics[width=0.96\linewidth]{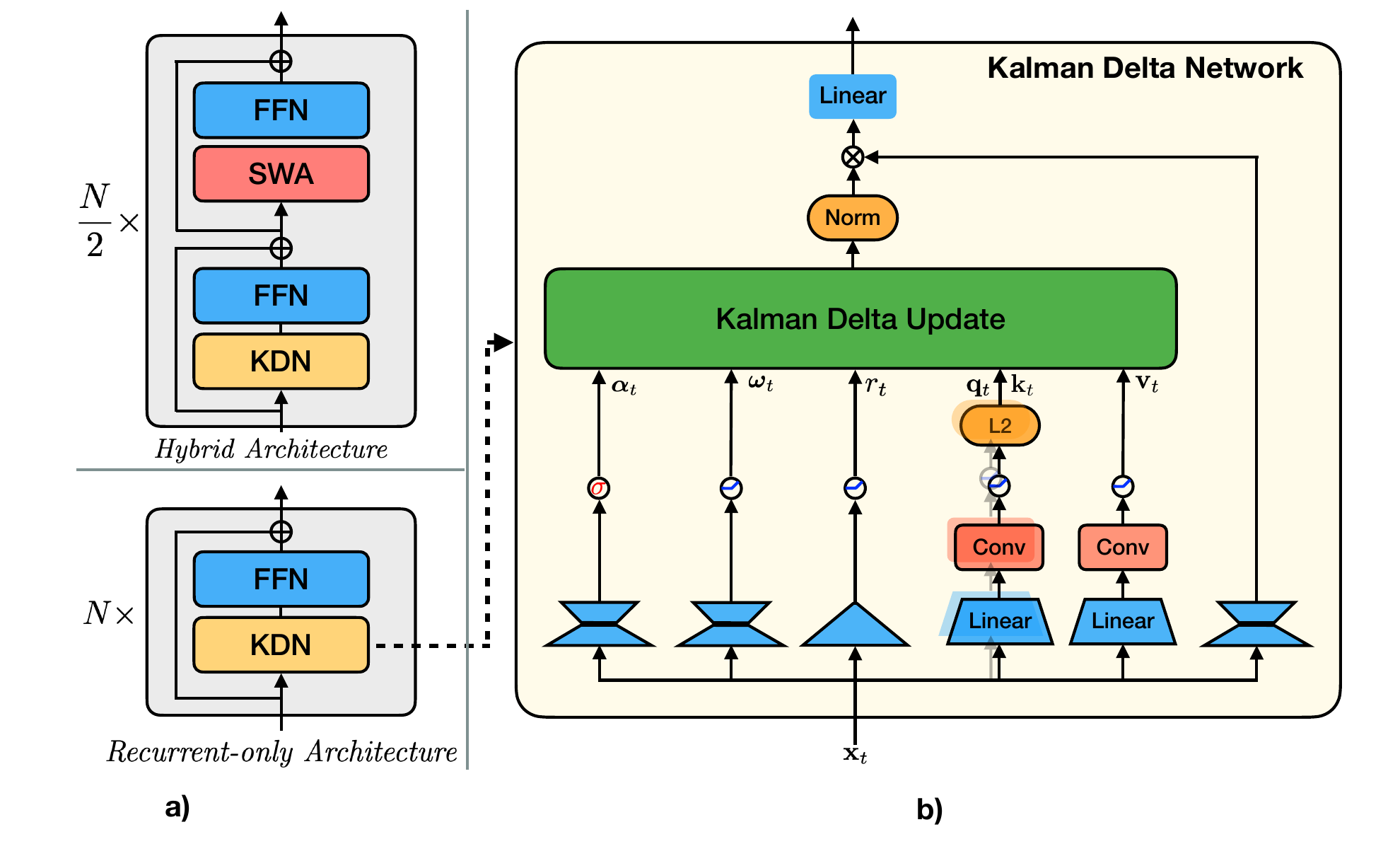}
  \vspace{-3mm}
  \caption{Kalman Delta Network architectures.}
  \label{fig:kdn-layer-architecture}
  \vspace{-4mm}
\end{wrapfigure}
\paragraph{Pretraining and evaluation datasets.}
We pretrain all models on FineWeb-Edu~\citep{penedo2024fineweb} at two scales:
$750$M parameters on $50$B tokens and $1.3$B parameters on $100$B tokens. Evaluations cover
language modeling, zero-shot commonsense reasoning, and synthetic and real-world retrieval.
Appendix~\ref{sec:experimental-setup} provides more details for the evaluation datasets.

\paragraph{Model architectures.}
KDN builds on KDA~\citep{kimiteam2025kimilinear}, replacing its $\beta_t$ gate with process
noise $\ve{\omega}_t$ and observation noise $r_t$. We use the recurrent-only and hybrid
backbones of GDN-2~\citep{hatamizadeh2026gdn2}, shown in
Figure~\ref{fig:kdn-layer-architecture}, with matched model capacity. At $1.3$B, hybrids
alternate recurrent layers with $2$K sliding-window attention (SWA).
Appendix~\ref{sec:experimental-setup} gives the dimensions and parameter- and state-matching details.

\paragraph{Baselines.}
We compare with DeltaNet~\citep{yang2024parallelizing}, Gated DeltaNet~\citep{yang2024gdn},
KDA~\citep{kimiteam2025kimilinear}, the SISO and MIMO variants of
Mamba-3~\citep{lahoti2026mamba3}, and GDN-2~\citep{hatamizadeh2026gdn2}. Our models are the Isotropic
and Diagonal KDN variants developed above. The hybrid block also includes an attention-only Transformer with a
$2$K sliding window.


\begin{table}[t]
\centering
\footnotesize
\resizebox{\linewidth}{!}{
\setlength{\tabcolsep}{3pt}
\begin{tabular}{@{}l cccc cccc ccc ccc@{}}
\toprule
& \multicolumn{4}{c}{S-NIAH-1} & \multicolumn{4}{c}{S-NIAH-2} & \multicolumn{3}{c}{S-NIAH-3} & \multicolumn{3}{c}{MK-NIAH-1} \\
\cmidrule(lr){2-5}\cmidrule(lr){6-9}\cmidrule(lr){10-12}\cmidrule(lr){13-15}
Model & 1K & 2K & 4K & 8K & 1K & 2K & 4K & 8K & 1K & 2K & 4K & 1K & 2K & 4K \\
\midrule
\multicolumn{15}{@{}l}{\emph{Recurrent-only, 750M parameters, 50B tokens}}\\
DeltaNet~\citeyearpar{yang2024parallelizing}       & \textbf{100.0} & \textbf{100.0} & \underline{99.6} & \underline{99.6} & 97.0 & 87.0 & 45.4 & 15.4 & 72.0 & 49.2 & 17.6 & 27.2 & 27.4 & 22.6 \\
Gated DeltaNet~\citeyearpar{yang2024gdn}           & \textbf{100.0} & \underline{99.8} & 99.2 & 87.8 & 93.2 & 56.2 & 27.2 & 13.8 & 71.8 & 23.8 & 6.6 & 29.8 & 25.4 & 20.2 \\
KDA~\citeyearpar{kimiteam2025kimilinear}           & \textbf{100.0} & \underline{99.8} & 99.2 & 82.6 & 96.4 & 89.8 & \underline{63.2} & \underline{26.0} & 68.8 & 43.6 & 13.6 & 32.6 & 29.8 & 24.2 \\
Mamba-3 (SISO)~\citeyearpar{lahoti2026mamba3} & \textbf{100.0} & 98.8 & 73.8 & 33.2 & 95.4 & 82.2 & 43.0 & 17.6 & 65.6 & 26.4 & 9.0 & \underline{42.4} & 27.6 & 19.6 \\
Mamba-3 (MIMO)~\citeyearpar{lahoti2026mamba3} & \textbf{100.0} & 99.0 & 79.2 & 39.6 & \underline{98.4} & \textbf{92.6} & 57.6 & 21.6 & 77.4 & 49.0 & 14.2 & \textbf{50.2} & \underline{36.0} & \underline{25.0} \\
GDN-2~\citeyearpar{hatamizadeh2026gdn2}         & \textbf{100.0} & \underline{99.8} & 90.2 & 47.2 & 95.0 & 74.6 & 36.0 & 12.6 & 64.6 & 31.0 & 9.2 & 27.8 & 25.6 & 22.0 \\
\rowcolor{blue!8}
Isotropic KDN & \textbf{100.0} & \textbf{100.0} & 98.2 & 89.4 & 96.0 & 86.0 & 47.8 & 16.2 & \underline{80.4} & \underline{58.4} & \underline{22.6} & 30.4 & 25.4 & 22.2 \\
\rowcolor{blue!8}
Diagonal KDN  & \textbf{100.0} & \textbf{100.0} & \textbf{100.0} & \textbf{100.0} & \textbf{98.6} & \underline{91.4} & \textbf{71.2} & \textbf{29.0} & \textbf{84.2} & \textbf{71.0} & \textbf{32.2} & 41.0 & \textbf{36.4} & \textbf{29.6} \\
\midrule
\multicolumn{15}{@{}l}{\emph{Recurrent-only, 1.3B parameters, 100B tokens}}\\
KDA~\citeyearpar{kimiteam2025kimilinear}           & \underline{99.8} & 99.2 & 66.6 & 30.2 & 98.0 & 95.0 & 62.2 & 23.6 & \underline{92.6} & \underline{77.2} & \underline{43.2} & 47.4 & 38.8 & 26.6 \\
Mamba-3 (SISO)~\citeyearpar{lahoti2026mamba3} & \textbf{100.0} & \underline{99.8} & 98.2 & 62.6 & \underline{98.6} & 93.6 & 62.6 & 20.6 & 69.0 & 50.2 & 16.6 & 46.6 & 38.4 & 27.2 \\
Mamba-3 (MIMO)~\citeyearpar{lahoti2026mamba3} & \underline{99.8} & 99.4 & 75.2 & 15.4 & \textbf{99.6} & \textbf{96.0} & 64.8 & 15.4 & 85.8 & 63.6 & 29.8 & \underline{51.6} & 42.4 & 30.4 \\
GDN-2~\citeyearpar{hatamizadeh2026gdn2}         & \textbf{100.0} & \textbf{100.0} & \textbf{100.0} & \textbf{99.8} & \underline{98.6} & 92.4 & 67.4 & \textbf{28.0} & 88.0 & 64.4 & 29.4 & 50.6 & \underline{44.0} & \underline{31.4} \\
\rowcolor{blue!8}
Isotropic KDN & \textbf{100.0} & \textbf{100.0} & \textbf{100.0} & \underline{94.6} & 98.4 & 91.8 & \underline{68.4} & 24.2 & 86.8 & 67.6 & 34.6 & 41.6 & 29.8 & 23.4 \\
\rowcolor{blue!8}
Diagonal KDN  & \textbf{100.0} & \textbf{100.0} & \underline{99.8} & \textbf{99.8} & \underline{98.6} & \underline{95.6} & \textbf{74.6} & \underline{25.8} & \textbf{96.4} & \textbf{88.6} & \textbf{48.6} & \textbf{62.2} & \textbf{47.4} & \textbf{33.8} \\
\midrule
\multicolumn{15}{@{}l}{\emph{Hybrid and attention-only, 1.3B parameters, 100B tokens}}\\
Transformer~\citeyearpar{vaswani2017attention} (2K SWA)
  & \textbf{100.0} & \textbf{100.0} & \textbf{52.4} & 22.6
  & \textbf{100.0} & \textbf{100.0} & \textbf{53.4} & 19.8
  & \underline{99.2} & 94.6 & 35.4
  & 76.6 & 79.4 & 45.4 \\
KDA~\citeyearpar{kimiteam2025kimilinear} $+$ SWA
  & \textbf{100.0} & 87.8 & 47.4 & 21.2
  & \underline{99.8} & \textbf{100.0} & \textbf{53.4} & 25.4
  & \underline{99.2} & 97.0 & 48.6
  & 86.0 & 83.0 & 44.2 \\
Mamba-3 (SISO)~\citeyearpar{lahoti2026mamba3} $+$ SWA
  & \underline{99.4} & 73.8 & 29.4 & 15.0
  & 99.6 & 99.2 & 52.2 & 21.8
  & 96.4 & 86.2 & 43.4
  & 74.8 & 76.0 & 39.0 \\
Mamba-3 (MIMO)~\citeyearpar{lahoti2026mamba3} $+$ SWA
  & \textbf{100.0} & \underline{99.8} & \underline{52.0} & 27.2
  & \underline{99.8} & \underline{99.8} & \underline{53.2} & 25.6
  & 97.8 & 95.8 & \underline{50.8}
  & \underline{92.8} & \underline{89.6} & 45.8 \\
GDN-2~\citeyearpar{hatamizadeh2026gdn2} $+$ SWA
  & \textbf{100.0} & \textbf{100.0} & \textbf{52.4} & \textbf{28.4}
  & \textbf{100.0} & \underline{99.8} & \textbf{53.4} & \textbf{27.6}
  & 83.0 & 67.8 & 25.2
  & 84.0 & 76.8 & 45.6 \\
\rowcolor{blue!8}
Isotropic KDN $+$ SWA
  & \textbf{100.0} & \textbf{100.0} & \textbf{52.4} & \textbf{28.4}
  & \textbf{100.0} & \textbf{100.0} & \textbf{53.4} & \textbf{27.6}
  & \textbf{99.8} & \underline{98.4} & 44.6
  & 85.4 & 89.0 & \textbf{49.8} \\
\rowcolor{blue!8}
Diagonal KDN $+$ SWA
  & \textbf{100.0} & 98.6 & 51.6 & \underline{28.0}
  & \textbf{100.0} & 98.4 & \textbf{53.4} & \underline{26.0}
  & \textbf{99.8} & \textbf{99.2} & \textbf{53.0}
  & \textbf{95.6} & \textbf{90.0} & \underline{47.8} \\
\bottomrule
\end{tabular}
}
\caption{In-context retrieval accuracy (\%) on RULER single- and multi-key needle-in-a-haystack
tasks~\citep{hsieh2024ruler}. Best per column within each block in
\textbf{bold}, second best \underline{underlined}.}
\label{tab:pretrain-niah}
\end{table}

\subsection{Experimental Results}
\label{sec:experimental-results}

\paragraph{Language modeling and commonsense reasoning.}
Table~\ref{tab:pretrain-lm} shows consistent gains from KDN across both parameter scales. Among
recurrent-only models, both Diagonal and Isotropic KDN outperform KDA and Mamba-3 on the two
perplexity benchmarks and in average accuracy at both scales, with Diagonal KDN leading in average
accuracy. In the hybrid setting, Diagonal KDN $+$ SWA achieves the highest average accuracy,
followed by KDA $+$ SWA, while every recurrent--attention hybrid outperforms the attention-only
Transformer.

\paragraph{In-context retrieval.}
Tables~\ref{tab:pretrain-niah} and~\ref{tab:real-world-retrieval} evaluate synthetic
needle-in-a-haystack and real-world retrieval, respectively.
On NIAH, Diagonal KDN
achieves the highest aggregate among recurrent-only models at both scales, with particularly strong
S-NIAH-3 and multi-key retrieval, thanks to channel-wise uncertainty that protects stored
associations under interference. On real-world tasks, Diagonal KDN
also achieves the best recurrent-only average and leads FDA, TriviaQA, and DROP. Among hybrids,
Isotropic KDN $+$ SWA has the best average, while Diagonal KDN $+$ SWA ranks second and leads FDA.
Additionally, every hybrid model substantially outperforms its recurrent-only counterpart on these
tasks.

\begin{table}[t]
\centering
\footnotesize
\setlength{\tabcolsep}{4.5pt}
\begin{tabular}{@{}l cccccc >{\columncolor{gray!10}}c@{}}
\toprule
Model & SWDE & SQuAD & FDA & TriviaQA & NQ & DROP & Avg. \\
\midrule
\multicolumn{8}{@{}l}{\emph{Recurrent 1.3B / 100B}}\\
KDA~\citeyearpar{kimiteam2025kimilinear}
  & 28.77 & \textbf{38.70} & \underline{27.97} & 61.97
  & \textbf{24.11} & 21.03 & \underline{33.76} \\
Mamba-3 (SISO)~\citeyearpar{lahoti2026mamba3}
  & 26.34 & 36.35 & 22.52 & 60.90 & 21.98 & 20.89 & 31.50 \\
Mamba-3 (MIMO)~\citeyearpar{lahoti2026mamba3}
  & 24.93 & 36.71 & 25.25 & 61.02 & 23.50 & 21.99 & 32.23 \\
GDN-2~\citeyearpar{hatamizadeh2026gdn2}
  & 29.90 & 35.64 & 21.16 & 61.49 & 23.34 & 20.99 & 32.09 \\
\rowcolor{blue!8}
Isotropic KDN
  & \textbf{33.08} & 36.71 & 22.98 & \underline{62.09}
  & 22.14 & \underline{23.38} & 33.40 \\
\rowcolor{blue!8}
Diagonal KDN
  & \underline{30.18} & \underline{38.09} & \textbf{30.61}
  & \textbf{62.56} & \underline{23.63} & \textbf{24.10}
  & \textbf{34.86} \\
\midrule
\multicolumn{8}{@{}l}{\emph{Hybrid 1.3B / 100B}}\\
Transformer~\citeyearpar{vaswani2017attention} (2K SWA)
  & 36.08 & 42.43 & 52.68 & 62.09 & 25.40 & 21.47 & 40.02 \\
KDA~\citeyearpar{kimiteam2025kimilinear} $+$ SWA
  & \underline{51.08} & 43.10 & 57.86 & 65.17
  & \underline{28.22} & \textbf{25.11} & 45.09 \\
Mamba-3 (SISO)~\citeyearpar{lahoti2026mamba3} $+$ SWA
  & 36.27 & 43.37 & 58.95 & 64.57 & 26.70 & 22.47 & 42.06 \\
Mamba-3 (MIMO)~\citeyearpar{lahoti2026mamba3} $+$ SWA
  & 42.55 & 43.33 & \underline{64.85} & \textbf{65.88}
  & \textbf{28.44} & \underline{23.53} & 44.76 \\
GDN-2~\citeyearpar{hatamizadeh2026gdn2} $+$ SWA
  & 47.80 & 42.73 & 63.22 & 62.38 & 26.10 & 23.19 & 44.24 \\
\rowcolor{blue!8}
Isotropic KDN $+$ SWA
  & \textbf{53.98} & \textbf{43.60} & 63.12 & \underline{65.52}
  & 26.89 & 22.52 & \textbf{45.94} \\
\rowcolor{blue!8}
Diagonal KDN $+$ SWA
  & 48.08 & \underline{43.50} & \textbf{66.58} & 64.93
  & 27.72 & 22.52 & \underline{45.55} \\
\bottomrule
\end{tabular}
\caption{Zero-shot accuracy (\%) on real-world retrieval tasks with inputs limited to $2$K tokens.
\emph{Avg.} is the unweighted six-task mean. Best per column within each block in \textbf{bold},
second best \underline{underlined}.}
\label{tab:real-world-retrieval}
\end{table}

\FloatBarrier

\subsection{Ablation and Runtime}
\label{sec:ablation-runtime}

\paragraph{Information scale.}
\begin{wraptable}{r}{0.37\textwidth}
  \vspace{-4mm}
  \centering
  \footnotesize
  \setlength{\tabcolsep}{4.5pt}
  \resizebox{\linewidth}{!}{%
  \begin{tabular}{@{}lccc@{}}
    \toprule
    $\mu$ & Wiki. PPL $\downarrow$ & LMB. PPL $\downarrow$ & Avg. $\uparrow$ \\
    \midrule
    $1$           & 18.69 & 15.51 & 54.12 \\
    $\sqrt{d_k}$  & 18.93 & 15.10 & 54.20 \\
    $d_k$         & 18.64 & 14.15 & \textbf{54.97} \\
    $4d_k$        & 18.97 & \textbf{13.91} & 54.92 \\
    Learned       & \textbf{18.51} & 14.68 & 54.62 \\
    \bottomrule
  \end{tabular}%
  }
  \vspace{-2mm}
  \caption{Information-scale ablation.}
  \label{tab:info-scale-ablation}
  \vspace{-0.6\baselineskip}
  \vspace{-2mm}
\end{wraptable}
Table~\ref{tab:info-scale-ablation} ablates the information scale using fixed values
$\mu\in\{1,\sqrt{d_k},d_k,4d_k\}$ and a learned scale initialized at $d_k$. Increasing the fixed scale
steadily improves LAMBADA perplexity, with $4d_k$ performing best, whereas WikiText
perplexity and average accuracy peak at $d_k$. The learned scale performs best on WikiText but
trails the fixed $d_k$ setting on LAMBADA and average accuracy, indicating that the overall effect
is small and metric-dependent.

\FloatBarrier

\paragraph{Throughput.}
\begin{wrapfigure}{r}{0.43\textwidth}
  \vspace{-6mm}
  \centering
  \includegraphics[width=\linewidth]{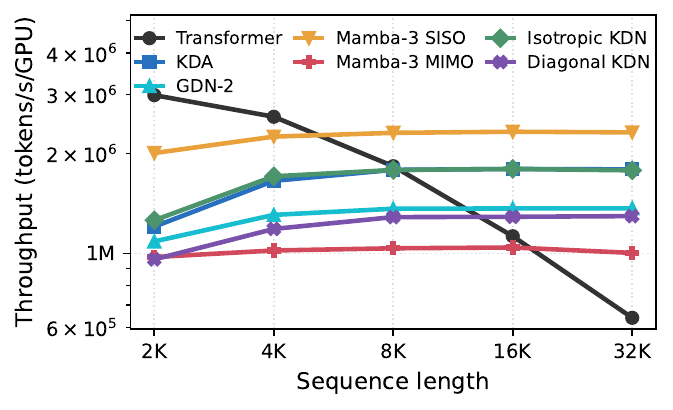}
  \vspace{-8mm}
  \caption{Mixer-layer throughput scaling.}
  \label{fig:mixer-throughput}
  \vspace{-0.6\baselineskip}
\end{wrapfigure}
Figure~\ref{fig:mixer-throughput} benchmarks the mixer layers' throughput using
the state-matched 750M configuration on one H200 ($B=4$, $d_{\mathrm{model}}=2048$; median
over 30 runs). Transformer uses full causal attention, and
lengths beyond the 4K training context are scaling probes. Isotropic KDN closely tracks KDA across
sequence lengths. Diagonal KDN incurs additional cost from the channel-wise uncertainty scan but
remains close to GDN-2 and retains linear scaling.
Full attention degrades sharply as sequence length grows.

\FloatBarrier

\vspace{-1mm}
\section{Related Work}
\vspace{-1mm}
\label{sec:related-work}

\paragraph{Fixed-memory attention.}
Standard self-attention~\citep{vaswani2017attention}, including its RoPE-based variants~\citep{su2024roformer}, retains token-level history and computes quadratic pairwise interactions. Practical systems reduce these costs through sparse attention~\citep{liu2025deepseek} or bounded KV-cache retention at inference time~\citep{xiao2024efficient,bui2026cache,bui2026make}. Efficient native alternatives include low-rank
projection in Linformer~\citep{wang2020linformer} and recurrent linear attention, which compresses
history into a fixed-size fast-weight state for constant-memory decoding and scan-parallel
training~\citep{katharopoulos2020transformers}. Gated and corrective recurrent designs include GLA,
DeltaNet, Gated DeltaNet, Comba, RWKV-7, and KDA
\citep{yang2023gated,schlag2021linear,yang2024gdn,hu2025comba,
peng2025rwkv7,kimiteam2025kimilinear}; FoX, DeltaFormer, and PaTH-FoX carry related forgetting,
associative correction, and multiplicative positional encoding to exponential-kernel attention
\citep{lin2025forgetting,zhong2025understanding,yang2025path}. GDN-2 further decouples erase and
write~\citep{hatamizadeh2026gdn2}. Practical Preconditioned DeltaNet variants~\citep{tumma2026preconditioned}
combine diagonal curvature with token-predicted gains, whereas KDN derives gains from propagated
uncertainty and observation noise.

\paragraph{State-space sequence mixers and Bayesian memory.}
Structured SSMs provide a complementary approach to fixed-memory sequence mixing: S4 makes long-range continuous-time dynamics
tractable via structured convolutions, while S5 recasts them as a scan-parallel MIMO
recurrence~\citep{gu2022efficiently,smith2023simplified},
Mamba adds input selectivity~\citep{gu2023mamba}, Mamba-2 derives attention--SSM duality and SSD
\citep{dao2024transformers}, and Mamba-3 introduces exponential--trapezoidal discretization, complex
dynamics, and SISO/MIMO variants~\citep{lahoti2026mamba3}. KLA~\citep{shaj2026kalmanlinearattention}
augments Mamba-style sequence mixing with a scan-parallel Kalman filter via belief factorization
and time-invariant OU dynamics. Unlike Mamba's deterministic, control-driven additive updates, KDN treats token values
as noisy observations of a stochastic key--value map and weights each innovation by propagated
covariance. Unlike KLA's factorized belief, KDN's exact posterior has dense key-space covariance
shared across value channels.
Concurrently, \citet{dowling2026memory} derive Bayesian Layers with dense covariance propagation.
However, their implementation requires quadratic covariance storage and sequential chunk processing,
whereas KDN uses scan-compatible isotropic and diagonal approximations.

\vspace{-1mm}
\section{Conclusion and Future Work}
\vspace{-1mm}
\label{sec:discussion-conclusion}

We introduced two new recurrent mixers, Isotropic and Diagonal KDNs, that generalize existing
delta-rule mixers by capturing uncertainty in stored associations to calibrate residual writes
while preserving parallel training. In parameter-matched recurrent-only pretraining at 750M and
1.3B parameters, both variants achieve lower WikiText and LAMBADA perplexity and higher mean
six-task zero-shot accuracy than all evaluated recurrent baselines. Diagonal KDN also achieves
the highest aggregate RULER score at both scales.

These models are steps toward, rather than full realizations of, Kalman Associative Memory. Exact
filtering maintains a dense key-space covariance and a state-dependent Riccati update, whereas the
isotropic and diagonal approximations compress this covariance to preserve parallel training. A future direction is to extend diagonal decay with the damped rotations used by Mamba-3
\citep{lahoti2026mamba3}, allowing stored associations to rotate as well as decay. Making this richer
transition and covariance update scan-efficient remains open.

\subsection*{AI use statement}

We used AI tools to polish the writing and assist with implementing and optimizing Triton kernels
for Isotropic KDN and Diagonal KDN. We implemented and verified recurrent Python reference
versions of both models, which served as correctness references and guardrails for AI agents
during kernel implementation and performance optimization. We take responsibility for the final
content of this work, including all AI-assisted text and code.

\subsection*{Reproducibility statement}

We specify the modeling assumptions and update equations in
Sections~\ref{sec:memory-kalman}--\ref{sec:kla}, with proofs and further derivations in
Appendices~\ref{app:proofs}, \ref{app:protection-write}, and~\ref{app:iso-kla}.
Appendix~\ref{sec:experimental-setup} documents the datasets, model configurations,
parameter- and state-matching procedures, training hyperparameters, and evaluation protocols.
Appendix~\ref{sec:chunkwise-algorithm} details the chunkwise gain and memory computations.
We also implemented and verified recurrent Python reference versions of both KDN variants to
serve as correctness references for Triton kernel development.

\bibliography{main}
\bibliographystyle{plainnat}

\appendix
\clearpage

\section{Additional Experiments}
\label{app:additional-experiments}

\subsection{Detailed Experimental Setup}
\label{sec:experimental-setup}

\paragraph{Pretraining and evaluation datasets.}
We pretrain all models from scratch on FineWeb-Edu~\citep{penedo2024fineweb} at two scales:
$750$M parameters on $50$B tokens and $1.3$B parameters on $100$B tokens.
\begin{itemize}
\item \textbf{Language Modeling.} We evaluate language
modeling with perplexity on WikiText~\citep{merity2017pointer} and
LAMBADA~\citep{paperno2016lambada}.
\item \textbf{Zero-shot reasoning.} We evaluate zero-shot reasoning with
LAMBADA~\citep{paperno2016lambada}, PIQA~\citep{bisk2020piqa},
HellaSwag~\citep{zellers2019hellaswag}, WinoGrande~\citep{sakaguchi2020winogrande}, and
ARC-Easy and ARC-Challenge~\citep{clark2018arc}.
\item \textbf{In-context retrieval.} We evaluate
in-context retrieval on the single- and multi-key needle-in-a-haystack tasks from
RULER~\citep{hsieh2024ruler}. We use $25$-word context increments and deterministic random essay
windows rather than $500$-word increments and fixed prefixes; every model receives the same $500$
samples per cell, with matched keys, values, and needle depths.
\item \textbf{Real-world retrieval.} Following Just Read Twice (JRT)~\citep{arora2024justreadtwice},
we additionally evaluate real-world retrieval on six cloze-formatted tasks: SWDE~\citep{lockard2019openceres},
SQuAD~\citep{rajpurkar2018squad}, FDA~\citep{arora2023evaporate},
TriviaQA~\citep{joshi2017triviaqa}, Natural Questions~\citep{kwiatkowski2019natural}, and
DROP~\citep{dua2019drop}.
\end{itemize}

\paragraph{Model architectures.}
The KDN mixer builds on KDA~\citep{kimiteam2025kimilinear}, replacing its $\beta_t$ gate with
process noise $\ve{\omega}_t$ and observation noise $r_t$. To isolate the mixer, we adopt the
recurrent-only and hybrid layer and model architectures of GDN-2~\citep{hatamizadeh2026gdn2}.
This controlled choice might not be optimal; determining the best recurrent-to-attention layer
ratio requires a dedicated ablation, as performed for KDA in Kimi
Linear~\citep{kimiteam2025kimilinear}.
Figure~\ref{fig:kdn-layer-architecture} summarizes the recurrent-only and hybrid designs. All
recurrent-only models use the same backbone, with feed-forward widths adjusted to match non-embedding
parameter counts within $0.03\%$. The $750$M
and $1.3$B backbones use 16 and 18 layers with $d_{\mathrm{model}}=2048$ and $2304$, respectively,
corresponding to KDN mixers with 16 and 18 recurrent heads of dimension $128$.

The Mamba-3 baselines are also
configured to match the KDNs' total recurrent-state size at both scales. At $1.3$B/$100$B, the
hybrid architecture alternates nine recurrent mixer layers with nine sliding-window-attention
(SWA) layers using a $2$K window.

\paragraph{Implementation details.}
We use the same training recipe as GDN-2~\citep{hatamizadeh2026gdn2} for all models: AdamW with a
peak learning rate of $4\times10^{-4}$, $\beta=(0.9,0.95)$, weight decay $0.1$, gradient clipping
at $1.0$, and a cosine schedule with a $1\%$ token-budget warm-up ($0.5$B/$1$B tokens at
the two scales). All runs use a global
batch of $0.5$M tokens, a sequence length of $4$K, and the same data seed within each comparison.
We set $r_{\min}=0.01$ and initialize the process noise isotropically by setting
$\ve{W}_{\omega}=\ve{0}$ when training Diagonal KDN. We also learn a scalar initial covariance
$c_{0,h}>0$ for each head $h$, so that
$\widetilde{\ve{S}}_0^{(h)}\sim\N_{\mathrm{col}}(\ve{0},c_{0,h}\ve{I}_{d_k})$.
Throughout the main experiments, both Isotropic and Diagonal KDN use the fixed information scale
$\mu=d_k$.

\subsection{Additional Experiment Results}
\label{app:additional-experiment-results}

\paragraph{Length extrapolation.}
Following Mamba-3~\citep{lahoti2026mamba3}, we evaluate the recurrent-only $1.3$B models from
Table~\ref{tab:pretrain-lm} on WikiText while increasing the context limit from $1$K to $32$K,
eight times the $4$K training sequence length. As shown in
Figure~\ref{fig:length-extrapolation-recurrent}, every model improves on its $4$K result at each
longer context limit, with most gains saturating by $16$K and no abrupt extrapolation failure.
Diagonal KDN achieves the lowest word perplexity at every evaluated length, improving from $15.04$
at $4$K to $14.67$ at $32$K; Isotropic KDN is consistently second among the recurrent models.

\begin{figure}[H]
  \centering
  \includegraphics[width=0.82\textwidth]{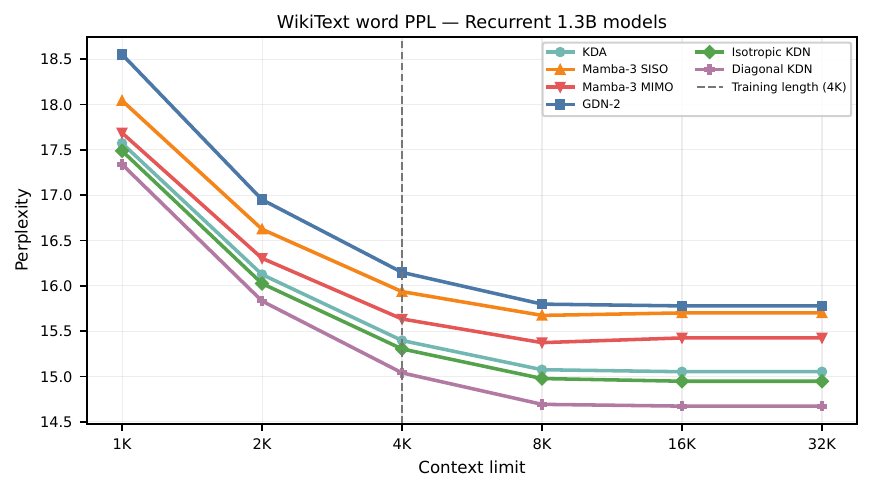}
  \vspace{-2mm}
  \caption{Length extrapolation of recurrent-only $1.3$B models trained with $4$K sequences.
  We report WikiText word perplexity (lower is better) as the context limit increases from $1$K to
  $32$K; the dashed line marks the training sequence length.}
  \label{fig:length-extrapolation-recurrent}
\end{figure}

\FloatBarrier

 
\paragraph{Observation and process noise.}

\begin{wraptable}{r}{0.49\textwidth}
  \centering
  \footnotesize
  \setlength{\tabcolsep}{2.5pt}
  \begin{tabular}{@{}lccc@{}}
    \toprule
    Setting ($\mu=1$) & Wiki. PPL $\downarrow$ & LMB. PPL $\downarrow$ & Avg. $\uparrow$ \\
    \midrule
    Learnable $r_t,\ve{\omega}_t$ & 18.69 & \textbf{15.51} & \textbf{54.12} \\
    Fix $r_t=1$ & \textbf{18.46} & 15.72 & 53.77 \\
    Fix both at $1$ & 18.64 & 15.84 & 53.87 \\
    \bottomrule
  \end{tabular}
  \caption{Observation- and process-noise ablation.}
  \label{tab:fixed-noise-ablation}
  \vspace{-0.6\baselineskip}
\end{wraptable}
Table~\ref{tab:fixed-noise-ablation} ablates the observation and process noise at $\mu=1$ by comparing learned $r_t$ and
$\ve{\omega}_t$ with variants that first fix $r_t=1$ and then fix both terms to one.
Fixing $r_t$ improves WikiText perplexity but slightly degrades LAMBADA perplexity and average accuracy; fixing
both terms yields similarly mixed results. Overall, neither constraint consistently improves
short-context quality.

\WFclear
\FloatBarrier

\section{Attention and State-Space Background}
\label{app:attention-ssm-background}

\paragraph{Self-attention.}
Causal softmax attention stores every preceding key--value pair and retrieves a query-dependent
weighted combination~\citep{vaswani2017attention},
\begin{equation}
  \ve{o}_t
  =\sum_{i=1}^{t}
    \frac{\exp(\ve{q}_t^{\tr}\ve{k}_i/\sqrt{d_k})}
    {\sum_{j=1}^{t}\exp(\ve{q}_t^{\tr}\ve{k}_j/\sqrt{d_k})}\,\ve{v}_i.
  \label{eq:softmax-attention}
\end{equation}
Retaining the complete token history provides direct access to past content, but the autoregressive
cache grows linearly with $t$, and processing a full sequence requires quadratic query--key
interactions.

\paragraph{Mamba state-space models.}
The Mamba family uses the selective linear state-space recurrence
\begin{equation}
  \ve{h}_t=\ve{A}_t\ve{h}_{t-1}+\ve{B}_t\ve{x}_t,
  \qquad
  \ve{o}_t=\ve{C}_t^{\tr}\ve{h}_t,
  \label{eq:mamba-ssm}
\end{equation}
where $\ve{A}_t$, $\ve{B}_t$, and $\ve{C}_t$ control retention, writing, and reading. Mamba-1 makes
these operations input-dependent~\citep{gu2023mamba}, while Mamba-2 connects the resulting
selective recurrence to structured masked attention through state-space
duality~\citep{dao2024transformers}. Mamba-3 adds exponential--trapezoidal discretization,
complex-valued dynamics, and SISO and MIMO variants~\citep{lahoti2026mamba3}. Unlike the delta-rule
family in \eqref{eq:unified-delta-rule}, Mamba adds the control $\ve{B}_t\ve{x}_t$ directly rather
than applying a key-conditioned residual correction.

\section{Theoretical Proofs}
\label{app:proofs}

\subsection{Proof of the Kalman optimal update}
\label{app:proof-kalman-optimal-update}

\begin{proof}[Proof of Proposition~\ref{prop:kalman-associative-memory} (Kalman Associative Memory)]

Let the latent associative-memory state be
\[
  \widetilde{\ve{S}}_t \in \R^{d_k\times d_v},
  \qquad
  \mathcal{M}_t(\ve{k})=\widetilde{\ve{S}}_t^{\tr}\ve{k}.
\]
The linear--Gaussian state-space model is
\begin{align}
  \widetilde{\ve{S}}_t
  &= \ve{D}_t\widetilde{\ve{S}}_{t-1}+\ve{W}_t,
  &
  \ve{W}_t&\sim\N_{\mathrm{col}}(\ve{0},\ve{\Omega}_t), \\
  \ve{v}_t
  &= \widetilde{\ve{S}}_t^{\tr}\ve{k}_t+\ve{e}_t,
  &
  \ve{e}_t&\sim\mathcal{N}(0,r_t\ve{I}_{d_v}).
\end{align}
The current model quantities $\ve{D}_t$, $\ve{\Omega}_t$, $r_t$, and $\ve{k}_t$ are treated as
known when the gain is selected.
Under these Gaussian assumptions, the filtering posterior at step $t-1$ is
\begin{equation}
  p(\widetilde{\ve{S}}_{t-1}\mid \mc F_{t-1})
  =
  \N_{\mathrm{col}}(\ve{S}_{t-1},\ve{P}_{t-1}),
  \label{eq:previous-filtering-posterior}
\end{equation}
where $\ve{S}_{t-1}$ is the posterior mean, i.e.\ the point estimate of the memory after
processing the first $t-1$ key--value pairs, and $\ve{P}_{t-1}$ is its key-space uncertainty.

Before observing token $t$, the filter first predicts the next latent memory by marginalizing
over the previous posterior:
\begin{equation}
  p(\widetilde{\ve{S}}_t\mid \mc F_{t-1})
  =
  \int
  p(\widetilde{\ve{S}}_t\mid \widetilde{\ve{S}}_{t-1})
  p(\widetilde{\ve{S}}_{t-1}\mid \mc F_{t-1})
  \,d\widetilde{\ve{S}}_{t-1} .
  \label{eq:prediction-integral}
\end{equation}
Because both factors are linear--Gaussian, the predicted distribution is Gaussian,
\begin{equation}
  p(\widetilde{\ve{S}}_t\mid \mc F_{t-1})
  =
  \N_{\mathrm{col}}(\widehat{\ve{S}}_t,\widehat{\ve{P}}_t).
  \label{eq:predicted-prior}
\end{equation}
Evaluating this Gaussian integral gives the exact Kalman prediction step:
\begin{equation}
  \widehat{\ve{S}}_t=\ve{D}_t\ve{S}_{t-1},
  \qquad
  \widehat{\ve{P}}_t
  =
  \ve{D}_t\ve{P}_{t-1}\ve{D}_t^{\tr}+\ve{\Omega}_t.
  \label{eq:kam-predict}
\end{equation}
At step $t$, we observe the key--value pair $(\ve{k}_t,\ve{v}_t)$, where
$\ve{k}_t^{\tr}$ is the known observation map.
Reading the current predicted memory at $\ve{k}_t$ gives the expected value
$\widehat{\ve{v}}_t=\widehat{\ve{S}}_t^{\tr}\ve{k}_t$. Since the actual observation is
$\ve{v}_t$, the innovation is its difference from this memory prediction:
\begin{equation}
  \ve{\delta}_t=\ve{v}_t-\widehat{\ve{v}}_t.
  \label{eq:kam-innovation}
\end{equation}

The columnwise predictive state and isotropic observation noise make the pairs
$(\widetilde{\ve{s}}_{t,j},v_{t,j})$ conditionally independent across $j$. Conditioned on
$(\mc F_{t-1},\ve{k}_t)$, each pair has joint distribution
\begin{equation}
  \begin{bmatrix}
    \widetilde{\ve{s}}_{t,j} \\ v_{t,j}
  \end{bmatrix}
  \sim
  \N\!\left(
  \begin{bmatrix}
    \widehat{\ve{s}}_{t,j} \\ \ve{k}_t^{\tr}\widehat{\ve{s}}_{t,j}
  \end{bmatrix},
  \begin{bmatrix}
    \widehat{\ve{P}}_t & \widehat{\ve{P}}_t\ve{k}_t \\
    \ve{k}_t^{\tr}\widehat{\ve{P}}_t & \gamma_t
  \end{bmatrix}
  \right),
  \qquad
  \gamma_t:=r_t+\ve{k}_t^{\tr}\widehat{\ve{P}}_t\ve{k}_t.
  \label{eq:kam-joint-column}
\end{equation}
Applying the Gaussian conditioning formula gives
\begin{equation}
\begin{aligned}
  \widetilde{\ve{s}}_{t,j}\mid\mc F_t
  &\sim
  \N\!\left(
    \widehat{\ve{s}}_{t,j}+\ve{\kappa}_t\delta_{t,j},
    \ve{P}_t
  \right), \\
  \ve{\kappa}_t
  &=\frac{\widehat{\ve{P}}_t\ve{k}_t}{\gamma_t},
  &
  \ve{P}_t
  &=\widehat{\ve{P}}_t
    -\frac{\widehat{\ve{P}}_t\ve{k}_t\ve{k}_t^{\tr}\widehat{\ve{P}}_t}
    {\gamma_t}.
\end{aligned}
  \label{eq:kam-column-conditioning}
\end{equation}
The posterior remains factorized, with a common gain and covariance across value coordinates.
Stacking the columnwise posteriors therefore yields
\begin{equation}
  \widetilde{\ve{S}}_t\mid\mc F_t
  \sim
  \N_{\mathrm{col}}\!\left(
    \widehat{\ve{S}}_t+\ve{\kappa}_t\ve{\delta}_t^{\tr},
    \ve{P}_t
  \right),
  \qquad
  \ve{P}_t
  =(\ve{I}-\ve{\kappa}_t\ve{k}_t^{\tr})\widehat{\ve{P}}_t.
  \label{eq:kam-stacked-posterior}
\end{equation}
This directly establishes the posterior mean and covariance.

It remains to verify the equivalent minimum-MSE characterization. For any candidate gain
$\ve{g}\in\mathbb{R}^{d_k}$, define
$\ve{S}_t(\ve{g})=\widehat{\ve{S}}_t+\ve{g}\ve{\delta}_t^{\tr}$ and
\begin{equation}
  J_t(\ve{g})
  :=\E\!\left[
    \left.
    \|\widetilde{\ve{S}}_t-\ve{S}_t(\ve{g})\|_F^2
    \,\right|\,\mc F_{t-1},\ve{k}_t
  \right].
  \label{eq:kam-mse-objective}
\end{equation}
Using the predictive covariance and independence of the observation noise, direct expansion and
completion of the square give
\begin{equation}
\begin{aligned}
  J_t(\ve{g})
  &=d_v\Big[
    \Trace(\widehat{\ve{P}}_t)
    -2\ve{g}^{\tr}\widehat{\ve{P}}_t\ve{k}_t
    +\gamma_t\ve{g}^{\tr}\ve{g}
  \Big] \\
  &=d_v\Big[
    \Trace(\ve{P}_t)
    +\gamma_t\|\ve{g}-\ve{\kappa}_t\|_2^2
  \Big].
\end{aligned}
  \label{eq:kam-mse-completed-square}
\end{equation}
Since $\gamma_t>0$, this objective has the unique minimizer
$\ve{g}=\ve{\kappa}_t$, establishing the minimum-MSE characterization.
Therefore the exact Kalman Associative Memory update is
\begin{equation}
\begin{aligned}
  \widehat{\ve{S}}_t
  &=\ve{D}_t\ve{S}_{t-1},
  &
  \widehat{\ve{P}}_t
  &=\ve{D}_t\ve{P}_{t-1}\ve{D}_t^{\tr}+\ve{\Omega}_t,
  \\
  \ve{\kappa}_t
  &=
  \frac{\widehat{\ve{P}}_t\ve{k}_t}
  {r_t+\ve{k}_t^{\tr}\widehat{\ve{P}}_t\ve{k}_t},
  &
  \ve{S}_t
  &=
  \widehat{\ve{S}}_t
  +
  \ve{\kappa}_t
  \big(\ve{v}_t-\widehat{\ve{S}}_t^{\tr}\ve{k}_t\big)^{\tr},
  \\
  \ve{P}_t
  &=
  (\ve{I}-\ve{\kappa}_t\ve{k}_t^{\tr})\widehat{\ve{P}}_t.
\end{aligned}
\label{eq:kam-optimal-update}
\end{equation}
\end{proof}

\subsection{Proof of the online diagonal variational update}
\label{app:proof-diag-variational-inference}

\begin{proof}[Proof of Proposition~\ref{prop:diag-variational-inference}
  (Online Diagonal Variational Update)]
Let $\ve{\Lambda}_t^\star=(\ve{P}_t^\star)^{-1}$. For
$q=\N_{\mathrm{col}}(\ve{S},\diag(\ve{p}))$, the columnwise factorization and the Gaussian KL formula give
\begin{equation}
\begin{aligned}
  &\operatorname{KL}\!\left(
    q\,\|\,\N_{\mathrm{col}}(\ve{S}_t^\star,\ve{P}_t^\star)
  \right) \\
  &\quad
    =\frac12\Trace\!\left[
      (\ve{S}-\ve{S}_t^\star)^{\tr}
      \ve{\Lambda}_t^\star
      (\ve{S}-\ve{S}_t^\star)
    \right]
    +\frac{d_v}{2}\sum_i\left[
      (\ve{\Lambda}_t^\star)_{ii}p_i
      -\log p_i
    \right]+C,
  \label{eq:proof-diag-vi-objective}
\end{aligned}
\end{equation}
where $C=\frac{d_v}{2}(\log\det\ve{P}_t^\star-d_k)$ is independent of $\ve{S}$ and $\ve{p}$.
Since $\ve{\Lambda}_t^\star\succ\ve{0}$, the mean term is uniquely minimized at
$\ve{S}=\ve{S}_t^\star$. For each $p_i>0$,
\begin{equation}
  \frac{\partial\operatorname{KL}}{\partial p_i}
  =\frac{d_v}{2}\left(
    (\ve{\Lambda}_t^\star)_{ii}-\frac{1}{p_i}
  \right),
  \qquad
  \frac{\partial^2\operatorname{KL}}{\partial p_i^2}
  =\frac{d_v}{2p_i^2}>0.
  \label{eq:proof-diag-vi-derivatives}
\end{equation}
Thus the unique minimizing variance is
$p_i=1/(\ve{\Lambda}_t^\star)_{ii}$. Finally, the exact posterior precision satisfies
\begin{equation}
  \ve{\Lambda}_t^\star
  =\widehat{\ve{P}}_t^{-1}
   +\frac{1}{r_t}\ve{k}_t\ve{k}_t^{\tr}.
  \label{eq:proof-diag-vi-exact-precision}
\end{equation}
Because $\widehat{\ve{P}}_t=\diag(\widehat{\ve{p}}_t)$,
\begin{equation}
  (\ve{\Lambda}_t^\star)_{ii}
  =\widehat p_{t,i}^{-1}+\frac{k_{t,i}^2}{r_t},
  \qquad
  p_{t,i}
  =\left(\widehat p_{t,i}^{-1}+\frac{k_{t,i}^2}{r_t}\right)^{-1}.
  \label{eq:proof-diag-vi-variance}
\end{equation}
Together with the exact Kalman posterior mean $\ve{S}_t^\star$ from
Proposition~\ref{prop:kalman-associative-memory}, this proves
\eqref{eq:diag-variational-solution}.
\end{proof}

\section{Overwrite and Protection--Write Decoupling}
\label{app:protection-write}

\paragraph{Empirical symptom.}
For any delta-style residual update, define the effective write strength
\begin{equation}
  \beta_t^{\mathrm{eff}}
  =\ve{k}_t^{\tr}\ve{\kappa}_t.
  \label{eq:appendix-effective-write}
\end{equation}
Reading the updated memory at the same key shows its operational meaning:
\begin{equation}
  \ve{v}_t-\ve{S}_t^{\tr}\ve{k}_t
  =\left(1-\beta_t^{\mathrm{eff}}\right)
   \left(\ve{v}_t-\widehat{\ve{S}}_t^{\tr}\ve{k}_t\right).
  \label{eq:appendix-effective-write-residual}
\end{equation}
Thus $\beta_t^{\mathrm{eff}}$ is the fraction of the current residual removed at its own key;
values near one correspond to an almost complete same-key rewrite. We call an update an
\emph{overwrite} when $\beta_t^{\mathrm{eff}}$ remains large despite strong prior evidence along
$\ve{k}_t$, causing a new observation to replace rather than cautiously revise the stored
association.

To analyze the write strength induced by KDN, we record $\beta_t^{\mathrm{eff}}$ across layers,
heads, and tokens on a fixed FineWeb-Edu stream for parameter-matched 750M KDA and Diagonal KDN
models, both at initialization and after 50B training tokens. The resulting distributions are shown
in Figure~\ref{fig:kdn-overwrite}. At $\mu=1$, Diagonal KDN concentrates its effective writes near
one, indicating frequent overwrite. Increasing $\mu$ from $1$ to $d_k$ shifts the distribution
toward lower effective write strength.

\paragraph{Information scale as future protection.}
To understand overwrite, consider a normalized key with equal energy on $m$ active channels and
isotropic predictive variance on that support: $k_i^2=1/m$ and $\widehat p_i=p$. Along this key,
the exact posterior and its unscaled diagonal projection retain variances
\begin{equation}
  V_{\mathrm{exact}}
  :=\ve{k}^{\tr}\ve{P}^{\star}\ve{k}
  =\frac{rp}{r+p},
  \qquad
  V_{\mathrm{proj}}
  :=\ve{k}^{\tr}\diag(\ve{p})\ve{k}
  =\frac{mrp}{mr+p}.
  \label{eq:appendix-balanced-key-mismatch}
\end{equation}
For $m>1$, $V_{\mathrm{proj}}>V_{\mathrm{exact}}$. Consider an immediate repeat of the same key
with no intervening covariance evolution, so that the retained variance $V$ becomes the predictive
variance for the next write. Its effective strength is
\begin{equation}
  \beta_{t+1}^{\mathrm{eff}}(V)
  =\frac{V}{r_{t+1}+V}.
  \label{eq:appendix-variance-write-link}
\end{equation}
Because this fraction increases with $V$, the diagonal projection assigns a stronger correction
to the repeated key than the exact posterior.

We introduce an information scale $\mu$ to reduce this retained uncertainty. For the causal
analysis, let it vary by token as $\mu_t$; the model uses a single constant $\mu$:
\begin{equation}
  p_{t,i}(\mu_t)
  =\left(\widehat p_{t,i}^{-1}
    +\mu_t\frac{k_{t,i}^2}{r_t}\right)^{-1}.
  \label{eq:appendix-protection-scale-update}
\end{equation}
For the balanced key above, the scaled projection retains
\begin{equation}
  V_{\mathrm{proj}}(\mu)=\frac{mrp}{mr+\mu p}.
  \label{eq:appendix-balanced-key-protection}
\end{equation}
Thus $\mu=m$ matches the exact one-step directional variance; for an equal-energy key with full
support, this gives $\mu=d_k$.

The gain $\ve{\kappa}_t$ is computed from the predictive state $\widehat{\ve{p}}_t$ before the
scaled uncertainty update, so $\mu_t$ does not change the current gain or memory estimate:
\begin{equation}
  \frac{\partial\ve{\kappa}_t}{\partial\mu_t}=\ve{0},
  \qquad
  \frac{\partial\ve{S}_t}{\partial\mu_t}=\ve{0},
  \qquad
  \frac{\partial p_{t,i}}{\partial\mu_t}
  =-\frac{k_{t,i}^2}{r_t}p_{t,i}^2\leq0.
  \label{eq:appendix-protection-current-invariance}
\end{equation}
Its first effect is on future gains. Holding the next token's parameters fixed and defining
$z_{t+1}=\sum_i\widehat p_{t+1,i}k_{t+1,i}^2$, we obtain
\begin{equation}
\begin{aligned}
  \frac{\partial z_{t+1}}{\partial\mu_t}
  &=-\frac{1}{r_t}\sum_i
    \alpha_{t+1,i}^2p_{t,i}^2k_{t,i}^2k_{t+1,i}^2\leq0, \\
  \frac{\partial\beta_{t+1}^{\mathrm{eff}}}{\partial\mu_t}
  &=\frac{r_{t+1}}{(r_{t+1}+z_{t+1})^2}
    \frac{\partial z_{t+1}}{\partial\mu_t}\leq0.
\end{aligned}
  \label{eq:appendix-protection-future-write}
\end{equation}
Thus $\mu$ is a protection scale: it leaves the current write unchanged while weakening future
writes along channels shared with the current key.

The matched 750M sweep exhibits the predicted reduction in write strength. At both initialization
and after training, the pooled mean $\beta^{\mathrm{eff}}$ decreases monotonically as $\mu$
increases:
\begin{table}[H]
  \centering
  \small
  \setlength{\tabcolsep}{8pt}
  \begin{tabular}{@{}lccc@{}}
    \toprule
    Information scale $\mu$ & Random init. & Trained at $\mu$ & Fixed-weight intervention \\
    \midrule
    $1$          & 0.842 & 0.943 & 0.886 \\
    $\sqrt{d_k}$ & 0.743 & 0.891 & 0.873 \\
    $d_k$        & 0.630 & 0.826 & 0.862 \\
    $4d_k$       & 0.580 & 0.805 & 0.857 \\
    \bottomrule
  \end{tabular}
  \caption{Mean effective write for the 750M Diagonal KDN ($d_k=128$). The first two columns use
  the matched random-initialization controls and checkpoints trained at each scale, evaluated on
  eight 2048-token sequences. The final column uses a separate 4096-token trajectory, changes only
  the runtime information scale of each trained checkpoint, and averages the resulting
  layer--head--token means across the four checkpoints; model weights are fixed, while downstream
  hidden states respond to the changed recurrence. Absolute levels should therefore be compared
  within, not across, these probe protocols. The intervention isolates the monotone protection
  effect from retraining compensation.}
  \label{tab:appendix-effective-write-scale}
\end{table}
The intervention reduces the mean from $0.886$ at $\mu=1$ to $0.857$ at $4d_k$, while the
matched trained models decrease from $0.943$ to $0.805$. Together with
Equation~\eqref{eq:appendix-protection-future-write}, this supports the interpretation of $\mu$ as
a future-write control rather than an increase in current observation strength.

\FloatBarrier

\section{Isotropic Kalman Delta Network}
\label{app:iso-kla}

\paragraph{Retained isotropic state.}
Following the notation of Section~\ref{sec:preliminaries},
$\widetilde{\ve{S}}_t$ is the latent random memory, while $\ve{S}_t$ and $\ve{P}_t$ are the retained
filtering mean and covariance. Isotropic KDN replaces the $d_k$ diagonal uncertainty values of
Diagonal KDN by one scalar per head:
\begin{equation}
  \ve{P}_t=b_t\ve{I},
  \qquad
  b_t>0.
  \label{eq:iso-cov-family}
\end{equation}

\paragraph{Two isotropic projections.}
The isotropic family is not closed under either stage of the Kalman recursion. Starting from
$\ve{P}_{t-1}=b_{t-1}\ve{I}$, a channel-wise transition
$\ve{D}_t=\diag(\ve{\alpha}_t)$ generally produces an anisotropic predictive covariance. Even
after projecting that prediction back to an isotropic covariance, conditioning on a key
$\ve{k}_t$ produces a dense posterior. Isotropic KDN therefore projects once after prediction and
once after conditioning.

For any $\ve{X}\in\mathbb{S}_{++}^{d_k}$, the Frobenius-nearest isotropic matrix is
\begin{equation}
  \Pi_{\mathrm{iso}}(\ve{X})
  =\argmin_{b\ve{I}:\,b>0}\frac12\|\ve{X}-b\ve{I}\|_F^2
  =\frac{1}{d_k}\Trace(\ve{X})\ve{I}.
  \label{eq:iso-trace-projection}
\end{equation}
Indeed,
$\|\ve{X}-b\ve{I}\|_F^2
=\|\ve{X}\|_F^2-2b\Trace(\ve{X})+d_kb^2$, so the projection preserves the average eigenvalue. We
apply this covariance projection after prediction, where the process update is affine. After
conditioning, we instead minimize reverse KL over isotropic Gaussians, which preserves the exact
posterior mean and retains the harmonic mean of the posterior covariance eigenvalues.

\paragraph{Predictive covariance.}
The memory mean follows the full channel-wise prediction
$\widehat{\ve{S}}_t=\ve{D}_t\ve{S}_{t-1}$. With isotropic process noise
$\ve{\Omega}_t=\omega_t\ve{I}$, the retained predictive covariance is
\begin{equation}
\begin{aligned}
  a_t
  &=\frac{1}{d_k}\Trace(\ve{D}_t^2)
   =\frac{1}{d_k}\sum_{i=1}^{d_k}\alpha_{t,i}^2,
  &
  \widehat b_t
  &=a_tb_{t-1}+\omega_t,
  \\
  \widehat{\ve{P}}_t
  &=\Pi_{\mathrm{iso}}\!\left(
      \ve{D}_t\ve{P}_{t-1}\ve{D}_t^{\tr}+\ve{\Omega}_t
    \right)
   =\widehat b_t\ve{I}.
\end{aligned}
  \label{eq:iso-kla-predict}
\end{equation}
Thus $\widehat b_t$ is the average predictive marginal variance. The projection is inactive when
$\ve{\alpha}_t$ is constant across channels, as in the scalar decay of Gated
DeltaNet~\citep{yang2024gdn}. For channel-wise decay as in KDA, the memory retains the full
transition while the uncertainty tracker keeps only its isotropic projection.

\paragraph{Posterior covariance.}
By the columnwise factorization, it suffices to consider value coordinate $j$. Let
$\widehat{\ve{s}}_{t,j}$ and $\ve{s}_{t,j}$ denote column $j$ of the predicted memory
$\widehat{\ve{S}}_t$ and filtered memory $\ve{S}_t$, respectively. Under the retained predictive
state, the one-step model is
\begin{equation}
  \widetilde{\ve{s}}_{t,j}\mid\mathcal{F}_{t-1}
  \sim\N(\widehat{\ve{s}}_{t,j},\widehat{\ve{P}}_t),
  \qquad
  v_{t,j}=\ve{k}_t^{\tr}\widetilde{\ve{s}}_{t,j}+e_{t,j},
  \qquad
  e_{t,j}\sim\N(0,r_t).
  \label{eq:iso-variational-model}
\end{equation}
The exact one-step filtering posterior under \eqref{eq:iso-variational-model}, treating
$\ve{k}_t$ and $r_t$ as known at step $t$, has the Kalman residual-update mean and a generally
dense covariance:
\begin{equation}
\begin{aligned}
  \ve{\kappa}_t
  &=\beta_t\ve{k}_t,
  &
  \beta_t
  &=\frac{\widehat b_t}{r_t+\widehat b_t\|\ve{k}_t\|_2^2},
  \\
  \ve{s}_{t,j}
  &=\widehat{\ve{s}}_{t,j}
    +\ve{\kappa}_t
      \left(v_{t,j}-\ve{k}_t^{\tr}\widehat{\ve{s}}_{t,j}\right),
  \\
  \ve{P}_t^\star
  &=(\ve{I}-\ve{\kappa}_t\ve{k}_t^{\tr})\widehat{\ve{P}}_t
  \\
  &=\widehat b_t\ve{I}
   -\frac{\widehat b_t^2\ve{k}_t\ve{k}_t^{\tr}}
    {r_t+\widehat b_t\|\ve{k}_t\|_2^2}.
\end{aligned}
  \label{eq:iso-exact-posterior}
\end{equation}
The star distinguishes this exact posterior covariance from the retained isotropic covariance
$\ve{P}_t$. Because $\ve{P}_t^\star$ is generally dense, we approximate the exact one-step
posterior within the isotropic family
\begin{equation*}
  \mathcal{Q}_{\mathrm{iso}}
  =\{\N(\ve{s},b\ve{I}):\ve{s}\in\R^{d_k},\ b>0\}.
\end{equation*}

\begin{proposition}[Online isotropic variational update]
\label{prop:iso-variational-inference}
Let
$\pi_{t,j}=p(\widetilde{\ve{s}}_{t,j}\mid
v_{t,j},\mathcal{F}_{t-1},\ve{k}_t,r_t)$ denote the exact one-step filtering posterior under
\eqref{eq:iso-variational-model}. The variational approximation is
\begin{equation}
  q_{t,j}
  =\argmin_{q\in\mathcal{Q}_{\mathrm{iso}}}
    \operatorname{KL}\!\left(q\,\|\,\pi_{t,j}\right).
  \label{eq:iso-variational-objective}
\end{equation}
Its unique solution is
\begin{equation}
\begin{aligned}
  q_{t,j}
  &=\N(\ve{s}_{t,j},b_t\ve{I}),
  \\
  \ve{s}_{t,j}
  &=\widehat{\ve{s}}_{t,j}
    +\ve{\kappa}_t
      \left(v_{t,j}-\ve{k}_t^{\tr}\widehat{\ve{s}}_{t,j}\right),
  \\
  b_t
  &=\left(
      \frac{1}{\widehat b_t}
      +\frac{\|\ve{k}_t\|_2^2}{d_k r_t}
    \right)^{-1}
   =\frac{\widehat b_t}
    {1+\widehat b_t\|\ve{k}_t\|_2^2/(d_k r_t)},
  \qquad
  \ve{\kappa}_t=\beta_t\ve{k}_t.
\end{aligned}
  \label{eq:iso-variational-solution}
\end{equation}
Thus the retained mean $\ve{s}_{t,j}$ is the exact one-step posterior mean, while $b_t\ve{I}$ is
the isotropic approximation to its covariance.
\end{proposition}

\begin{proof}
The exact posterior is Gaussian with mean $\ve{s}_{t,j}$ and covariance $\ve{P}_t^\star$ from
\eqref{eq:iso-exact-posterior}. Its covariance eigenvalue along $\ve{k}_t$ and its eigenvalue on
the orthogonal complement are, respectively,
\begin{equation}
  \lambda_{\parallel}
  =\frac{\widehat b_t r_t}
    {r_t+\widehat b_t\|\ve{k}_t\|_2^2},
  \qquad
  \lambda_{\perp}=\widehat b_t.
  \label{eq:iso-posterior-eigenvalues}
\end{equation}
For $q=\N(\ve{s},b\ve{I})$, minimizing the Gaussian KL first sets
$\ve{s}=\ve{s}_{t,j}$. Up to constants independent of $b$, the remaining objective is
\begin{equation}
  \frac{1}{2}
  \left[
    b\left(\lambda_{\parallel}^{-1}
      +(d_k-1)\lambda_{\perp}^{-1}\right)
    -d_k\log b
  \right].
  \label{eq:iso-variational-covariance-objective}
\end{equation}
Its unique minimizer is the harmonic mean of these covariance eigenvalues,
\begin{equation}
  b_t
  =\frac{d_k}
    {\lambda_{\parallel}^{-1}+(d_k-1)\lambda_{\perp}^{-1}}
  =\left(
      \frac{1}{\widehat b_t}
      +\frac{\|\ve{k}_t\|_2^2}{d_k r_t}
    \right)^{-1},
\end{equation}
which is \eqref{eq:iso-variational-solution}.
\end{proof}

Starting from the isotropic variational update above, we introduce an information scale $\mu>0$
by setting
$b_t^{-1}-\widehat b_t^{-1}=\mu\|\ve{k}_t\|_2^2/(d_k r_t)$; $\mu=1$ recovers the
variational update.
Like the per-channel diagonal recurrence, the resulting scalar recurrence is a M\"obius map and
therefore admits an associative prefix scan. Isotropic KDN is summarized below.

\refstepcounter{definition}\label{prop:isotropic-kalman-linear-attention}
\begin{tcolorbox}[
  title={Proposition~\thedefinition: Isotropic Kalman Delta Network},
  breakable,
  colback=blue!7,
  colbacktitle=blue!22,
  colframe=blue!45!black,
  coltitle=black,
  fonttitle=\bfseries,
  boxrule=0.4pt,
  arc=1pt,
  left=5pt,
  right=5pt,
  top=5pt,
  bottom=5pt
]
Let the predictive covariance be restricted to the isotropic family
$\{b\ve{I}:b>0\}$ and let
$\ve{D}_t$, $\ve{\Omega}_t$, and $r_t$ be parameterized by
\begin{equation}
\begin{aligned}
  \ve{\alpha}_t &= \sigma(\ve{W}_{\alpha}\ve{x}_t+\ve{b}_{\alpha}),
  &
  \omega_t &= \operatorname{softplus}(\ve{w}_{\omega}^{\tr}\ve{x}_t+b_\omega),
  &
  r_t &= r_{\min}+\operatorname{softplus}(\ve{w}_{r}^{\tr}\ve{x}_t+b_r),
  \\[2pt]
  \ve{D}_t &= \diag(\ve{\alpha}_t),
  &
  \ve{\Omega}_t &= \omega_t\ve{I}.
\end{aligned}
\label{eq:iso-kla-box-parameterization}
\end{equation}
Here $\ve{x}_t$ is the token hidden state, and all weights $\ve{W}_{\ast}$ and
$\ve{w}_{\ast}$, together with their biases, are learnable parameters.
Define the trace-averaged transition factor
$a_t=d_k^{-1}\sum_i\alpha_{t,i}^2$.
For an information scale $\mu>0$, define
\begin{equation}
  u_t=\mu\frac{\|\ve{k}_t\|_2^2}{d_k r_t},
  \qquad
  \widehat b_t
  =a_tb_{t-1}+\omega_t,
  \qquad
  b_t
  =\frac{\widehat b_t}{1+u_t\widehat b_t},
  \qquad b_0=1.
  \label{eq:iso-kla-box-cov-update}
\end{equation}
The covariance states are computed by the scalar M\"obius scan
\begin{equation}
  \ve{M}_t=
  \begin{bmatrix}
    a_t & \omega_t\\
    u_ta_t & 1+u_t\omega_t
  \end{bmatrix},
  \qquad
  \begin{bmatrix}n_t\\ d_t\end{bmatrix}
  =
  \ve{M}_t
  \begin{bmatrix}n_{t-1}\\ d_{t-1}\end{bmatrix},
  \qquad
  b_t=\frac{n_t}{d_t},
  \qquad
  \begin{bmatrix}n_0\\ d_0\end{bmatrix}
  =
  \begin{bmatrix}1\\ 1\end{bmatrix}.
  \label{eq:iso-kla-box-mobius}
\end{equation}
The adaptive gain at token $t$ uses the pre-write predictive covariance $\widehat b_t$:
\begin{equation}
  \ve{\kappa}_t
  =\frac{\widehat b_t\ve{k}_t}
    {r_t+\widehat b_t\|\ve{k}_t\|_2^2}
  \equiv\beta_t\ve{k}_t.
\label{eq:iso-kla-box-gain}
\end{equation}
With $\beta_t$ available before the memory scan, the memory update has the same affine form as KDA:
\begin{equation}
\begin{aligned}
  \ve{S}_t
  &=(\ve{I}-\beta_t\ve{k}_t\ve{k}_t^{\tr})\ve{D}_t\ve{S}_{t-1}
    +\beta_t\ve{k}_t\ve{v}_t^{\tr},
  &
  \ve{o}_t &= \ve{S}_t^{\tr}\ve{q}_t .
\end{aligned}
  \label{eq:iso-kla-box-memory}
\end{equation}
\end{tcolorbox}

At each step, the gain is formed from $\widehat b_t$ before $b_t$ is updated, so the token-$t$
information increment does not affect the token-$t$ write; it affects only subsequent gains.
In all reported Isotropic KDN runs, we fix the information scale to $\mu=d_k$.

Isotropic KDN therefore preserves the predict--update decomposition and Kalman residual write, but
compresses the uncertainty state to one scalar per head. Its gain is constrained to
$\ve{\kappa}_t=\beta_t\ve{k}_t$; Diagonal KDN replaces $b_t$ by the vector $\ve{p}_t$ and permits
an anisotropic gain direction.

\section{Hardware-Efficient Chunkwise Algorithm}
\label{sec:chunkwise-algorithm}

Diagonal KDN is evaluated in two stages. A chunked M\"obius scan first computes every adaptive
gain $\ve{\kappa}_t$ from the diagonal covariance recurrence. Once these gains are available,
the memory recurrence is an input-only asymmetric delta rule and can be evaluated with a compact
WY transform, following the chunkwise algorithms developed for DeltaNet, KDA, and
GDN-2~\citep{yang2024parallelizing,kimiteam2025kimilinear,hatamizadeh2026gdn2}. The two stages
therefore use a M\"obius scan and a WY memory scan, respectively; both expose parallel
chunk-local work and retain only a short recurrence across chunks.

\subsection{Chunkwise Gain Computation}
\label{sec:chunkwise-gain}

\paragraph{Associative covariance scan.}
Because the retained covariance is diagonal, each channel evolves independently. Substituting
\eqref{eq:diag-cov-update}, with $u_{t,i}=\mu k_{t,i}^2/r_t$, yields the M\"obius recurrence
\begin{equation}
  p_{t,i}
  =\frac{\alpha_{t,i}^2p_{t-1,i}+\omega_{t,i}}
    {u_{t,i}\alpha_{t,i}^2p_{t-1,i}+1+u_{t,i}\omega_{t,i}}.
  \label{eq:diag-cov-mobius}
\end{equation}
As shown in~\citep{shaj2026kalmanlinearattention}, representing each M\"obius update by a $2\times2$ matrix turns composition into matrix
multiplication and yields the associative scan
\begin{equation}
  \begin{gathered}
    \ve{M}_{t,i}
    =\begin{bmatrix}
      \alpha_{t,i}^2 & \omega_{t,i}\\
      u_{t,i}\alpha_{t,i}^2 & 1+u_{t,i}\omega_{t,i}
    \end{bmatrix},
    \quad
    \begin{bmatrix}n_{t,i}\\d_{t,i}\end{bmatrix}
    =\ve{M}_{t,i}
      \begin{bmatrix}n_{t-1,i}\\d_{t-1,i}\end{bmatrix},
    \quad
    p_{t,i}=\frac{n_{t,i}}{d_{t,i}},
    \quad
    \begin{bmatrix}n_{0,i}\\d_{0,i}\end{bmatrix}
    =\begin{bmatrix}1\\1\end{bmatrix}.
  \end{gathered}
  \label{eq:diag-cov-scan}
\end{equation}
Thus the approximation reduces the uncertainty state from a dense $d_k\times d_k$ matrix to
$d_k$ scalars and replaces the sequential Riccati recursion with an associative scan of
logarithmic parallel depth.

\paragraph{Chunkwise gain scan.}
Let $\ve{z}_{t,i}=[n_{t,i},d_{t,i}]^{\tr}$ be the homogeneous covariance state in
\eqref{eq:diag-cov-scan}. Partition the sequence
into gain chunks $\mathcal{I}_m=\{\ell_m,\ldots,r_m\}$ of length at most $C_g$. Within each chunk,
the token-local M\"obius matrices are reduced in temporal order to
\begin{equation}
  \overline{\ve{M}}_{m,i}
  =\ve{M}_{r_m,i}\ve{M}_{r_m-1,i}\cdots\ve{M}_{\ell_m,i}.
  \label{eq:chunk-gain-summary}
\end{equation}
An exclusive prefix scan over these summaries gives the state entering every chunk,
\begin{equation}
  \ve{z}^{\mathrm{in}}_{m,i}
  =\overline{\ve{M}}_{m-1,i}\cdots\overline{\ve{M}}_{0,i}\ve{z}_{0,i},
  \qquad
  \ve{z}_{0,i}=\begin{bmatrix}p_{0,i}\\1\end{bmatrix},
  \label{eq:chunk-gain-carry}
\end{equation}
where the product is the identity for the first chunk. Starting from this exclusive carry, all
chunks are replayed in parallel. For $t\in\mathcal{I}_m$, the state immediately before token $t$
is
\begin{equation}
  \ve{z}_{t-1,i}
  =\ve{M}_{t-1,i}\cdots\ve{M}_{\ell_m,i}\ve{z}^{\mathrm{in}}_{m,i},
  \label{eq:chunk-gain-local-prefix}
\end{equation}
with an empty product at $t=\ell_m$. The gain is emitted from this exclusive state before
$\ve{M}_{t,i}$ is applied:
\begin{equation}
\begin{aligned}
  \widehat p_{t,i}
  &=\alpha_{t,i}^2\frac{n_{t-1,i}}{d_{t-1,i}}+\omega_{t,i},
  &
  \ve{\kappa}_t
  &=\frac{\widehat{\ve{p}}_t\odot\ve{k}_t}
  {r_t+\sum_i\widehat p_{t,i}k_{t,i}^2},
  &
  \ve{z}_{t,i}&=\ve{M}_{t,i}\ve{z}_{t-1,i}.
\end{aligned}
  \label{eq:chunk-gain-emission}
\end{equation}
This gives a three-pass implementation: compute one $2\times2$ summary per chunk and channel,
scan the summaries to obtain exclusive chunk carries, and replay the chunks in parallel to emit
$\ve{\kappa}_t$. The ratio $n/d$ is invariant to a common rescaling, so after each multiplication
we divide every homogeneous vector or matrix by its largest-magnitude entry. This fp32
renormalization prevents overflow when $\omega_{t,i}$ is very small. Our gain kernel uses
$C_g=256$; the short carry pass is serial over the number of chunks, although the summaries also
admit a fully associative prefix scan.

\subsection{Chunkwise Memory Scan}
\label{sec:chunkwise-memory}

With $\ve{\kappa}_t$ fixed, partition the sequence independently into memory chunks of length
$C_m$. Consider one chunk, suppress its global index, and let $\ve{S}^{\mathrm{in}}$ be its
incoming memory. For local position $r$, define the coordinatewise cumulative decay
\begin{equation}
  \ve{\gamma}_r=\prod_{j=1}^{r}\ve{\alpha}_j,
  \qquad
  \ve{\gamma}_0=\ve{1},
  \qquad
  \ve{S}_r=\diag(\ve{\gamma}_r)\ve{R}_r.
  \label{eq:chunk-memory-normalization}
\end{equation}
The normalized state obeys an ordinary two-key delta rule,
\begin{equation}
  \ve{R}_r
  =\left(\ve{I}-\overline{\ve{\kappa}}_r\overline{\ve{k}}_r^{\tr}\right)
    \ve{R}_{r-1}
   +\overline{\ve{\kappa}}_r\ve{v}_r^{\tr},
  \qquad
  \overline{\ve{k}}_r=\ve{\gamma}_r\odot\ve{k}_r,
  \qquad
  \overline{\ve{\kappa}}_r=\ve{\kappa}_r\oslash\ve{\gamma}_r.
  \label{eq:chunk-memory-normalized-recurrence}
\end{equation}
Let $\ve{K}$, $\ve{K}^{(\kappa)}$, $\ve{Q}$, $\ve{V}$, and $\ve{\Gamma}$ stack the chunk's
rows $\ve{k}_r^{\tr}$, $\ve{\kappa}_r^{\tr}$, $\ve{q}_r^{\tr}$, $\ve{v}_r^{\tr}$, and
$\ve{\gamma}_r^{\tr}$, respectively. Let $\tau$ denote the query scale, with
$\tau=d_k^{-1/2}$ in our implementation and $\tau=1$ when it is absorbed into $\ve{q}_r$. Define
\begin{equation}
  \overline{\ve{K}}=\ve{\Gamma}\odot\ve{K},
  \qquad
  \overline{\ve{K}}^{(\kappa)}=\ve{K}^{(\kappa)}\oslash\ve{\Gamma},
  \qquad
  \overline{\ve{Q}}=\tau\ve{\Gamma}\odot\ve{Q}.
  \label{eq:chunk-memory-scaled-keys}
\end{equation}
The compact WY factors are
\begin{equation}
\begin{aligned}
  \ve{T}
  &=\operatorname{tril}\!\left(
      \overline{\ve{K}}\bigl(\overline{\ve{K}}^{(\kappa)}\bigr)^{\tr},-1
    \right),
  &
  \ve{A}&=(\ve{I}_{C_m}+\ve{T})^{-1},
  \\
  \ve{W}&=\ve{A}\overline{\ve{K}},
  &
  \ve{U}&=\ve{A}\ve{V},
  &
  \ve{R}&=\ve{U}-\ve{W}\ve{S}^{\mathrm{in}}.
\end{aligned}
  \label{eq:chunk-memory-wy}
\end{equation}
Here $\ve{T}$ is strictly lower triangular. The implementation computes and stores $\ve{A}$ with a
tiled unit-lower-triangular solve, then forms $\ve{W}$ and $\ve{U}$. Define the
final-decay-adjusted write matrix and the inclusive causal query--write matrix by
\begin{equation}
\begin{aligned}
  \left[\ve{K}_{\mathrm{tail}}^{(\kappa)}\right]_{r,:}
  &=\left[
      (\ve{\gamma}_{C_m}\oslash\ve{\gamma}_r)\odot\ve{\kappa}_r
    \right]^{\tr},
  &
  \ve{A}_{q\kappa}
  &=\operatorname{tril}\!\left(
      \overline{\ve{Q}}\bigl(\overline{\ve{K}}^{(\kappa)}\bigr)^{\tr}
    \right).
\end{aligned}
  \label{eq:chunk-memory-tail-output}
\end{equation}
The outputs for the whole chunk and its outgoing memory are then
\begin{equation}
\begin{aligned}
  \ve{O}
  &=\overline{\ve{Q}}\ve{S}^{\mathrm{in}}+\ve{A}_{q\kappa}\ve{R},
  &
  \ve{S}^{\mathrm{out}}
  &=\diag(\ve{\gamma}_{C_m})\ve{S}^{\mathrm{in}}
    +\bigl(\ve{K}_{\mathrm{tail}}^{(\kappa)}\bigr)^{\tr}\ve{R}.
\end{aligned}
  \label{eq:chunk-memory-output-state}
\end{equation}
The diagonal of $\ve{A}_{q\kappa}$ is included because $\ve{o}_t$ reads the state after the
current write. These equations are exact for the Diagonal KDN recurrence. They reduce to KDA when
$\ve{\kappa}_t=\beta_t\ve{k}_t$ and otherwise use the same two-key WY structure as GDN-2, with
$\ve{\kappa}_t$ as the write direction and $\ve{k}_t$ as the erase direction. All chunk-local
factors are computed in parallel; only $\ve{S}^{\mathrm{in}}$ is carried across memory chunks.
The implementation forms decay ratios from cumulative log-decay differences for numerical
stability and uses $C_m\in\{32,64\}$, independently of the gain-chunk length.

\end{document}

%% file: preamble.tex
\usepackage{graphicx}
\usepackage{float,epstopdf}
\usepackage{bbm}

\usepackage{microtype}

\usepackage{natbib}
\setcitestyle{square}

\usepackage{subcaption}
\usepackage{booktabs}

\usepackage{amsmath}
\usepackage{amssymb}
\usepackage{mathtools}
\usepackage{amsthm}
\usepackage{dsfont}
\usepackage{multicol}
\usepackage{makecell}
\usepackage{multirow} 
\usepackage{amsfonts} 
\usepackage{mathrsfs}
\usepackage[amssymb, thickqspace]{SIunits}
\usepackage{enumitem}
\usepackage{pgfplotstable}
\usepackage{lipsum}		

\usepackage{microtype}
\usepackage{graphicx}
\usepackage{booktabs} 
\usepackage[table]{xcolor}
\usepackage{arydshln}
\usepackage[normalem]{ulem} 

\usepackage{cases}
\usepackage{wrapfig}

\usepackage{url}

\usepackage{thmtools}
\usepackage{thm-restate}
\usepackage{tabu}

\definecolor{huskypurple}{HTML}{4B2E83}

\usepackage{titletoc}

\usepackage{listings}
\lstdefinestyle{promptstyle}{
  basicstyle=\ttfamily\footnotesize,
  breaklines=true,
  breakautoindent=false,
  breakindent=0pt,
  postbreak=\mbox{\textcolor{gray}{$\hookrightarrow$}\space},
  columns=fullflexible,
  keepspaces=true,
  frame=single,
  framesep=5pt,
  xleftmargin=6pt,
  xrightmargin=6pt,
  aboveskip=8pt,
  belowskip=8pt,
  showstringspaces=false,
}

%% file: main.bib
@String(AAAI = {AAAI})

@inproceedings{hsieh2024ruler,
title={{RULER}: What{\textquoteright}s the Real Context Size of Your Long-Context Language Models?},
author={Cheng-Ping Hsieh and Simeng Sun and Samuel Kriman and Shantanu Acharya and Dima Rekesh and Fei Jia and Boris Ginsburg},
booktitle={First Conference on Language Modeling},
year={2024},
}

@article{yang2024parallelizing,
  title={Parallelizing linear transformers with the delta rule over sequence length},
  author={Yang, Songlin and Wang, Bailin and Zhang, Yu and Shen, Yikang and Kim, Yoon},
  journal={arXiv preprint arXiv:2406.06484},
  year={2024}
}

@article{shaj2026kalmanlinearattention,
  title={Kalman Linear Attention: Parallel Bayesian Filtering For Efficient Language Modelling and State Tracking},
  author={Shaj, Vaisakh and Barker, Cameron and Scannell, Aidan and Szecsenyi, Andras and Crowley, Elliot J. and Storkey, Amos},
  journal={arXiv preprint arXiv:2602.10743},
  year={2026}
}

@inproceedings{katharopoulos2020transformers,
  title={Transformers are rnns: Fast autoregressive transformers with linear attention},
  author={Katharopoulos, Angelos and Vyas, Apoorv and Pappas, Nikolaos and Fleuret, Fran{\c{c}}ois},
  booktitle={International conference on machine learning},
  pages={5156--5165},
  year={2020},
  organization={PMLR}
}

@article{wang2020linformer,
  title={Linformer: Self-attention with linear complexity},
  author={Wang, Sinong and Li, Belinda Z and Khabsa, Madian and Fang, Han and Ma, Hao},
  journal={arXiv preprint arXiv:2006.04768},
  year={2020}
}

@article{yang2023gated,
  title={Gated linear attention transformers with hardware-efficient training},
  author={Yang, Songlin and Wang, Bailin and Shen, Yikang and Panda, Rameswar and Kim, Yoon},
  journal={arXiv preprint arXiv:2312.06635},
  year={2023}
}

@article{su2024roformer,
  title={Roformer: Enhanced transformer with rotary position embedding},
  author={Su, Jianlin and Ahmed, Murtadha and Lu, Yu and Pan, Shengfeng and Bo, Wen and Liu, Yunfeng},
  journal={Neurocomputing},
  volume={568},
  pages={127063},
  year={2024},
  publisher={Elsevier}
}

@article{vaswani2017attention,
  title={Attention is all you need},
  author={Vaswani, Ashish and Shazeer, Noam and Parmar, Niki and Uszkoreit, Jakob and Jones, Llion and Gomez, Aidan N and Kaiser, {\L}ukasz and Polosukhin, Illia},
  journal={Advances in neural information processing systems},
  volume={30},
  year={2017}
}

@inproceedings{schlag2021linear,
  title={Linear transformers are secretly fast weight programmers},
  author={Schlag, Imanol and Irie, Kazuki and Schmidhuber, J{\"u}rgen},
  booktitle={International Conference on Machine Learning (ICML)},
  year={2021},
  note={arXiv:2102.11174}
}

@article{yang2024gdn,
  title={Gated delta networks: Improving Mamba2 with delta rule},
  author={Yang, Songlin and Kautz, Jan and Hatamizadeh, Ali},
  journal={arXiv preprint arXiv:2412.06464},
  year={2024}
}

@inproceedings{gu2022efficiently,
  title={Efficiently Modeling Long Sequences with Structured State Spaces},
  author={Gu, Albert and Goel, Karan and R{\'e}, Christopher},
  booktitle={International Conference on Learning Representations},
  year={2022}
}

@inproceedings{smith2023simplified,
  title={Simplified State Space Layers for Sequence Modeling},
  author={Smith, Jimmy T. H. and Warrington, Andrew and Linderman, Scott W.},
  booktitle={International Conference on Learning Representations},
  year={2023}
}

@article{gu2023mamba,
  title={{Mamba}: Linear-Time Sequence Modeling with Selective State Spaces},
  author={Gu, Albert and Dao, Tri},
  journal={arXiv preprint arXiv:2312.00752},
  year={2023}
}

@article{dao2024transformers,
  title={Transformers are SSMs: Generalized Models and Efficient Algorithms Through Structured State Space Duality},
  author={Dao, Tri and Gu, Albert},
  journal={arXiv preprint arXiv:2405.21060},
  year={2024}
}

@article{penedo2024fineweb,
  title={The FineWeb Datasets: Decanting the Web for the Finest Text Data at Scale},
  author={Penedo, Guilherme and Kydl{\'i}{\v c}ek, Hynek and Lozhkov, Anton and Mitchell, Margaret and Raffel, Colin and von Werra, Leandro and Wolf, Thomas},
  journal={arXiv preprint arXiv:2406.17557},
  year={2024}
}

@inproceedings{lahoti2026mamba3,
  title={Mamba-3: Improved Sequence Modeling Using State Space Principles},
  author={Lahoti, Aakash and Li, Kevin Y. and Chen, Berlin and Wang, Caitlin and Bick, Aviv and Kolter, J. Zico and Dao, Tri and Gu, Albert},
  booktitle={The Fourteenth International Conference on Learning Representations},
  year={2026}
}

@article{hatamizadeh2026gdn2,
  title={Gated DeltaNet-2: Decoupling erase and write in linear attention},
  author={Hatamizadeh, Ali and Choi, Yejin and Kautz, Jan},
  journal={arXiv preprint arXiv:2605.22791},
  year={2026}
}

@article{kimiteam2025kimilinear,
  title={Kimi Linear: An Expressive, Efficient Attention Architecture},
  author={{Kimi Team}},
  journal={arXiv preprint arXiv:2510.26692},
  year={2025}
}

@article{lin2025forgetting,
  title={Forgetting Transformer: Softmax Attention with a Forget Gate},
  author={Lin, Zhixuan and others},
  journal={arXiv preprint arXiv:2503.02130},
  year={2025}
}

@article{zhong2025understanding,
  title={Understanding Transformer from the Perspective of Associative Memory},
  author={Zhong, Shu and others},
  journal={arXiv preprint arXiv:2505.19488},
  year={2025}
}

@article{yang2025path,
  title={{PaTH Attention}: Position Encoding via Accumulating Householder Transformations},
  author={Yang, Songlin and others},
  journal={arXiv preprint arXiv:2505.16381},
  year={2025}
}

@article{hu2025comba,
  title={{Comba}: Improving Bilinear {RNNs} with Closed-loop Control},
  author={Hu, Jiaxi and others},
  journal={arXiv preprint arXiv:2506.02475},
  year={2025}
}

@article{peng2025rwkv7,
  title={{RWKV-7} ``Goose'' with Expressive Dynamic State Evolution},
  author={Peng, Bo and others},
  journal={arXiv preprint arXiv:2503.14456},
  year={2025}
}

@article{kalman1960new,
  title={A New Approach to Linear Filtering and Prediction Problems},
  author={Kalman, Rudolf E.},
  journal={Journal of Basic Engineering},
  volume={82},
  number={1},
  pages={35--45},
  year={1960},
  doi={10.1115/1.3662552}
}

@misc{arora2024justreadtwice,
  title={Just Read Twice: Closing the Recall Gap for Recurrent Language Models},
  author={Arora, Simran and Timalsina, Aman and Singhal, Aaryan and Spector, Benjamin and Eyuboglu, Sabri and Zhao, Xinyi and Rao, Ashish and Rudra, Atri and R{\'e}, Christopher},
  year={2024},
  eprint={2407.05483},
  archivePrefix={arXiv},
  primaryClass={cs.CL},
  url={https://arxiv.org/abs/2407.05483}
}

@article{bui2026make,
  title={Make Each Token Count: Towards Improving Long-Context Performance with KV Cache Eviction},
  author={Bui, Ngoc and Nguyen, Hieu Trung and Cohan, Arman and Ying, Rex},
  journal={arXiv preprint arXiv:2605.09649},
  year={2026}
}

@inproceedings{xiao2024efficient,
  title={Efficient streaming language models with attention sinks},
  author={Xiao, Guangxuan and Tian, Yuandong and Chen, Beidi and Han, Song and Lewis, Mike},
  booktitle={International Conference on Learning Representations},
  volume={2024},
  pages={21875--21895},
  year={2024}
}

@inproceedings{bui2026cache,
  title={Cache what lasts: Token retention for memory-bounded kv cache in llms},
  author={Bui, Ngoc and Sharma, Shubham and Lamba, Simran and Mishra, Saumitra and Ying, Rex},
  booktitle={International Conference on Learning Representations},
  volume={2026},
  pages={120316--120342},
  year={2026}
}

@article{liu2025deepseek,
  title={Deepseek-v3. 2: Pushing the frontier of open large language models},
  author={Liu, Aixin and Mei, Aoxue and Lin, Bangcai and Xue, Bing and Wang, Bingxuan and Xu, Bingzheng and Wu, Bochao and Zhang, Bowei and Lin, Chaofan and Dong, Chen and others},
  journal={arXiv preprint arXiv:2512.02556},
  year={2025}
}

@article{tumma2026preconditioned,
  title={Preconditioned {DeltaNet}: Curvature-aware Sequence Modeling for Linear Recurrences},
  author={Tumma, Neehal and Loo, Noel and Rus, Daniela},
  journal={arXiv preprint arXiv:2604.21100},
  year={2026}
}

@inproceedings{merity2017pointer,
  title={Pointer Sentinel Mixture Models},
  author={Merity, Stephen and Xiong, Caiming and Bradbury, James and Socher, Richard},
  booktitle={International Conference on Learning Representations},
  year={2017}
}

@inproceedings{paperno2016lambada,
  title={The {LAMBADA} Dataset: Word Prediction Requiring a Broad Discourse Context},
  author={Paperno, Denis and Kruszewski, Germ{\'a}n and Lazaridou, Angeliki and Pham, Ngoc Quan and Bernardi, Raffaella and Pezzelle, Sandro and Baroni, Marco and Boleda, Gemma and Fern{\'a}ndez, Raquel},
  booktitle={Proceedings of the 54th Annual Meeting of the Association for Computational Linguistics},
  pages={1525--1534},
  year={2016}
}

@inproceedings{bisk2020piqa,
  title={{PIQA}: Reasoning about Physical Commonsense in Natural Language},
  author={Bisk, Yonatan and Zellers, Rowan and Le Bras, Ronan and Gao, Jianfeng and Choi, Yejin},
  booktitle={Proceedings of the AAAI Conference on Artificial Intelligence},
  pages={7432--7439},
  year={2020}
}

@inproceedings{zellers2019hellaswag,
  title={{HellaSwag}: Can a Machine Really Finish Your Sentence?},
  author={Zellers, Rowan and Holtzman, Ari and Bisk, Yonatan and Farhadi, Ali and Choi, Yejin},
  booktitle={Proceedings of the 57th Annual Meeting of the Association for Computational Linguistics},
  pages={4791--4800},
  year={2019}
}

@inproceedings{sakaguchi2020winogrande,
  title={{WinoGrande}: An Adversarial Winograd Schema Challenge at Scale},
  author={Sakaguchi, Keisuke and Le Bras, Ronan and Bhagavatula, Chandra and Choi, Yejin},
  booktitle={Proceedings of the AAAI Conference on Artificial Intelligence},
  pages={8732--8740},
  year={2020}
}

@article{clark2018arc,
  title={Think You Have Solved Question Answering? Try {ARC}, the {AI2} Reasoning Challenge},
  author={Clark, Peter and Cowhey, Isaac and Etzioni, Oren and Khot, Tushar and Sabharwal, Ashish and Schoenick, Carissa and Tafjord, Oyvind},
  journal={arXiv preprint arXiv:1803.05457},
  year={2018}
}

@inproceedings{lockard2019openceres,
  title={{OpenCeres}: When Open Information Extraction Meets the Semi-Structured Web},
  author={Lockard, Colin and Shiralkar, Prashant and Dong, Xin Luna},
  booktitle={Proceedings of the 2019 Conference of the North American Chapter of the Association for Computational Linguistics},
  pages={3047--3056},
  year={2019}
}

@article{arora2023evaporate,
  title={Language Models Enable Simple Systems for Generating Structured Views of Heterogeneous Data Lakes},
  author={Arora, Simran and Yang, Brandon and Eyuboglu, Sabri and Narayan, Avanika and Hojel, Andrew and Trummer, Immanuel and R{\'e}, Christopher},
  journal={Proceedings of the VLDB Endowment},
  year={2023}
}

@inproceedings{rajpurkar2018squad,
  title={Know What You Don't Know: Unanswerable Questions for {SQuAD}},
  author={Rajpurkar, Pranav and Jia, Robin and Liang, Percy},
  booktitle={Proceedings of the 56th Annual Meeting of the Association for Computational Linguistics},
  pages={784--789},
  year={2018}
}

@inproceedings{joshi2017triviaqa,
  title={{TriviaQA}: A Large Scale Distantly Supervised Challenge Dataset for Reading Comprehension},
  author={Joshi, Mandar and Choi, Eunsol and Weld, Daniel and Zettlemoyer, Luke},
  booktitle={Proceedings of the 55th Annual Meeting of the Association for Computational Linguistics},
  pages={1601--1611},
  year={2017}
}

@article{kwiatkowski2019natural,
  title={Natural Questions: A Benchmark for Question Answering Research},
  author={Kwiatkowski, Tom and Palomaki, Jennimaria and Redfield, Olivia and Collins, Michael and Parikh, Ankur and Alberti, Chris and Epstein, Danielle and Polosukhin, Illia and Devlin, Jacob and Lee, Kenton and Toutanova, Kristina and Jones, Llion and Kelcey, Matthew and Chang, Ming-Wei and Dai, Andrew M. and Uszkoreit, Jakob and Le, Quoc and Petrov, Slav},
  journal={Transactions of the Association for Computational Linguistics},
  volume={7},
  pages={452--466},
  year={2019}
}

@inproceedings{dua2019drop,
  title={{DROP}: A Reading Comprehension Benchmark Requiring Discrete Reasoning over Paragraphs},
  author={Dua, Dheeru and Wang, Yizhong and Dasigi, Pradeep and Stanovsky, Gabriel and Singh, Sameer and Gardner, Matt},
  booktitle={Proceedings of the 2019 Conference of the North American Chapter of the Association for Computational Linguistics},
  pages={2368--2378},
  year={2019}
}

@article{dowling2026memory,
  title={Memory by Design: Probabilistic Sequence Layers},
  author={Dowling, Matthew and Jeon, Hyungju and Savin, Cristina and Park, Il Memming},
  journal={arXiv preprint arXiv:2605.31163},
  year={2026},
  url={https://arxiv.org/abs/2605.31163}
}
